\documentclass[preprint]{sig-alternate}

\toappear{To appear in the 20th ACM SIGKDD Conference on Knowledge Discovery and Data Mining, August 2014, New York, NY, USA.}

\usepackage{amsmath,amssymb}
\usepackage{braket}
\usepackage{graphicx, subcaption}
\usepackage{enumitem}
\usepackage{algorithm, algorithmic}
\usepackage{url, hyperref}

\DeclareMathOperator*{\argmin}{arg\,min}
\DeclareMathOperator*{\argmax}{arg\,max}
\DeclareMathOperator{\Tr}{Tr}

\newcommand{\mbb}[1]{\boldsymbol{#1}}

\usepackage{bm}

\long\def\comment#1{}

\newfont{\bbb}{msbm10 scaled 700}

\newcommand{\fv}{{\bf f}}

\newcommand{\tv}{{\bf t}}
\newcommand{\uv}{{\bf u}}

\newcommand{\xv}{{\bf x}}

\newcommand{\zerov}{{\bf 0}}
\newcommand{\onev}{{\bf 1}}

\newcommand{\Am}{{\bf A}}

\newcommand{\Hm}{{\bf H}}

\newcommand{\Lm}{{\bf L}}
\newcommand{\Mm}{{\bf M}}

\newcommand{\Um}{{\bf U}}
\newcommand{\Wm}{{\bf W}}

\newcommand{\Kc}{{\cal K}}

\newcommand{\Pc}{{\cal P}}

\newcommand{\Sc}{{\cal S}}
\newcommand{\Tc}{{\cal T}}

\newcommand{\Vc}{{\cal V}}

\newcommand{\Lcb}{{\bm {\mathcal L}}}

\newcommand{\phiv}{\hbox{\boldmath$\phi$}}

\newcommand{\mb}[1]{\mathbf{#1}}
\newcommand{\mc}[1]{\mathcal{#1}}
\newcommand{\mcb}[1]{\boldsymbol{\mathcal{#1}}}
\newcommand{\defeq}{\stackrel{\triangle}{=}}

\newcommand{\Scc}{{\Sc^c}}
\newcommand{\Ds}{{\bm {\mathcal{D}}}_{\mathcal{S}}}

\newcommand{\Lcm}{\mcb{L}}
\newcommand{\Dcm}{\mcb{D}}

\newtheorem{theorem}{Theorem}

\begin{document}

%

\title{Active Semi-Supervised Learning Using Sampling Theory for Graph Signals
}

%
%
%
%
%

\numberofauthors{1} 
\author{\alignauthor Akshay Gadde, Aamir Anis and Antonio Ortega \\
\affaddr{Department of Electrical Engineering} \\ 
\affaddr{ University of Southern California, Los Angeles} \\
\email{agadde@usc.edu, aanis@usc.edu, ortega@sipi.usc.edu}
} 


\date{\today}

\maketitle


\begin{abstract}
We consider the problem of offline, pool-based active semi-supervised learning on graphs. 
This problem is important when the labeled data is scarce and expensive whereas unlabeled data is easily available. 
The data points are represented by the vertices of an undirected graph with the similarity between them captured by the edge weights.
Given a target number of nodes to label, the goal is to choose those nodes that are most informative and then predict the unknown labels. 
We propose a novel framework for this problem based on our recent results on sampling theory for graph signals. 
A graph signal is a real-valued function defined on each node of the graph.
A notion of frequency for such signals can be defined using the spectrum of the graph Laplacian matrix. 
The sampling theory for graph signals aims to extend the traditional Nyquist-Shannon sampling theory by allowing us to identify the class of graph signals that can be reconstructed from their values on a subset of vertices. 
This approach allows us to define a criterion for active learning based on sampling set selection which aims at maximizing the frequency of the signals that can be reconstructed from their samples on the set.
Experiments show the effectiveness of our method.
\end{abstract}





\newpage
\section{Introduction}
\vspace{0.5\baselineskip}


In many real-life machine learning tasks, labeled data is scarce whereas unlabeled data is easily available. Active semi-supervised learning is an effective approach for such scenarios. A semi-supervised learning technique must not only learn from the labeled data but also from the inherent clustering present in the unlabeled data~\cite{Zhu-UWM-08}. Further, when the labeling is expensive, it is better to let the learner choose the data points to be labeled so that it can pick the most informative and representative labels. Thus, in an active learning scenario, the goal is to achieve the maximum gain in terms of learning ability for a given, and small, number of label queries. 
In this paper, we propose a novel approach to active semi-supervised learning based on recent advances in sampling theory for graph signals. 

Active learning has been studied in different problem scenarios such as online stream-based sampling, adaptive sampling etc. (see~\cite{Settles-UWM-10} for a review). 
We focus on the problem of pool-based batch-mode active semi-supervised learning, 
where there is a large static collection of unlabeled data from which a very small percentage of data points 
have to be selected in order to be labeled.
%
Batch operation (i.e., selecting a \emph{set} of data points to be labeled) is more realistic in scenarios such as crowdsourcing where it would not be practical to submit for labeling one data point at a time. 
Further, in this paper we focus on the problem of optimizing batches of any size without using any label information, which would be the case when selecting the first batch of 
data points to be labeled.
We leave for future work the problem of incorporating labeled data, which would allow labels obtained for the first batch to be used to optimize data point selection for the second batch, and so on. 


Applying a graph perspective to semi-supervised learning is not new. 
In a graph-based formulation, the data points are represented by nodes of a graph and the edges capture the similarity between the nodes they connect. For example, the weight on an edge might be a function of the distance between the two points in the feature space chosen
for the classification task.  
The membership function of a given class can be thought of as a ``graph signal'', which has a scalar value at each of the nodes (e.g., 1 or 0 depending on whether or not the data point belongs to the class).
Since features have been chosen to be meaningful for the classification task, it is reasonable to expect that nodes that are close together in the feature space will be likely to have the same label. Conversely, nodes that are far away in the feature space are less likely to have the same label. Thus, we expect the membership function to be \emph{smooth} on the graph, i.e., moving from a node to its neighbors in the graph is unlikely to lead to changes in the membership. Thus, the semi-supervised learning problem can be viewed as a problem of interpolating a smooth graph signal. 
This view has led to many effective techniques such as MinCut~\cite{Blum-ICML-01}, Gaussian random fields and harmonic functions~\cite{Zhu-SSL-ICML-03}, local and global consistency~\cite{Zhou-NIPS-04}, manifold regularization~\cite{Belkin-JMLR-06} and spectral graph kernels~\cite{Smola-KER-03}.

Active learning has also benefited from this graph based-view. Many active learning approaches use the graph to quantify the quality of sampling sets~\cite{Gu-ICDM-12, Guillory-NIPS-09, Guillory-UAI-11}. One methodology is to try and pick  a subset of nodes which captures the underlying low-dimensional manifold represented by the graph. Another is to pick the nodes to be labeled in such a way that unlabeled nodes are strongly connected to them. Some methods pick those samples which lead to minimization of generalization error bound. We discuss some of these methods in Section~\ref{sec:related_work}.  

Many of the semi-supervised methods mentioned above are \emph{global}, in the sense that they require inversion or eigen-decomposition of large matrices associated with the underlying graph. This poses a problem in scalable and distributed implementation of these algorithms. Most graph-based active learning methods suffer from the same problem. Another issue with these methods is that they do not give conditions under which the graph signal can be uniquely and perfectly interpolated from its samples on the chosen subset.  

In recent years, there has been a significant amount of work devoted to the study of graph signal processing. The focus of this work has been to extend to the context of graph signals, theoretical results and tools that are well established in the context of conventional signal processing  \cite{SPM_Graphs}. 
In particular, there have been contributions to the design of graph wavelets \cite{hammond11}, 
graph filterbanks \cite{sunil12}, etc.
A key challenge in graph signal processing is to design \emph{localized} algorithms that scale well with graph sizes, i.e., the output at each vertex should only depend on its local neighborhood.

In this paper we leverage our recent work on graph signal sampling and interpolation \cite{Narang-GlobalSIP-13,Anis-ICASSP-14}. We show that the newly developed theoretical results provide a rigorous and unified framework to select points to be labeled and subsequently perform semi-supervised learning. Our framework provides conditions under which a graph signal can be uniquely recovered from its values on a subset of vertices. These conditions lead to a powerful greedy algorithm for choosing the best nodes for labeling. The proposed algorithm is well motivated through a compelling graph theoretic interpretation. We give a numerically efficient way to implement the proposed algorithm which makes it scalable. We also give an effective and efficient semi-supervised learning method that is closely tied to the label selection algorithm and is theoretically well-justified. Both our algorithms are well-suited for a large-scale distributed implementation. We show that our method outperforms several state of the art methods by testing on multiple real datasets. 

The rest of the paper is organized as follows. Section~\ref{sec:sampling_theory} reviews our recent work on sampling theory for graph signals. In Section~\ref{sec:active_ssl} we apply the framework of sampling theory to derive the proposed active semi-supervised learning approach. Section~\ref{sec:related_work} summarizes the related prior work. Experiments are presented in Section~\ref{sec:experiments}. Finally, we provide some concluding remarks in Section~\ref{sec:conclusion}. 

\vspace{0.5\baselineskip}
\section{Sampling theory for \\ graph signals}
\label{sec:sampling_theory}
\vspace{0.5\baselineskip}
We begin by briefly describing the theory of sampling for graph signals formulated in our previous work \cite{Narang-GlobalSIP-13, Anis-ICASSP-14}.

\subsection{Notation}
Throughout this paper, we consider simple, connected, undirected, and weighted graphs $G = (\Vc,E)$ with nodes numbered from the set $\Vc = \{1,2,\dots,N\}$, and edges $E = \{(i,j,w_{ij})\}, i,j \in \mc{V}$, where $(i,j,w_{ij})$ denotes an edge of weight $w_{ij}$ between nodes $i$ and $j$, with $w_{ii} = 0$.
In the present context, the weights denote similarity between the respective nodes.
The degree $d_i$ of a node $i$ is defined as the sum of the weights of edges connected to node $i$, and the degree matrix of the graph is a diagonal matrix defined as $\mb{D} = \text{diag}\{d_1,d_2,\dots,d_N\}$.
The adjacency matrix $\mb{W}$ of the graph is an $N\times N$ matrix with $\Wm_{ij} = w_{ij}$ and the combinatorial Laplacian matrix is defined as $\mb{L} = \mb{D} - \mb{W}$. 
We shall use the symmetric normalized form of the adjacency and the Laplacian matrices defined as $\mcb{W} = \mb{D}^{-1/2}\mb{W}\mb{D}^{-1/2}$ and $\mcb{L} = \mb{D}^{-1/2}\mb{L}\mb{D}^{-1/2}$ respectively. 
$\mcb{L}$ is a symmetric positive semi-definite matrix and has a set of real eigenvalues $0 = \lambda_1 \leq \lambda_2 \leq \dots \leq \lambda_N \leq 2$ and a corresponding orthogonal set of eigenvectors denoted as $\mb{U} = \{\uv_1,\uv_2,\dots,\uv_N\}$.
A subset of nodes of the graph is denoted as a collection of indices $\Sc \subset \Vc$, with $\Scc = \Vc \setminus \Sc$ denoting its complement set.
A restriction of a matrix $\Am$ to rows in set $\Sc_1$ and columns in set $\Sc_2$ is denoted by the submatrix $\Am_{\Sc_1,\Sc_2}$ and for the sake of brevity $\Am_{\Sc,\Sc} = \Am_\Sc$.
Also, $\zerov$ and $\onev$ denote all-zeroes and all-ones vectors of appropriate sizes.

A graph signal is defined as a scalar-valued discrete mapping $f:V \rightarrow \mathbb{R}$, such that $f(i)$ is the value of the signal on node $i$. 
For ease of notation, it can also be represented as a vector $\fv \in \mathbb{R}^N$ with indices corresponding to the node indices in the graph.
In this paper, the signals of interest will be the membership functions associated with the various labels of interest in the classification problem. 
\emph{Sampling} a graph signal $\fv$ onto a subset of nodes $\Sc$, known as the \emph{sampling set}, is realized by retaining the signal's values on the nodes in $\Sc$. The sampled signal is denoted by $\fv(\Sc)$, which is a vector of reduced length $|\Sc|$. 
In our context, a sampled graph signal will include the membership information for the data points that have been labeled.  


\subsection{Preliminaries}

The classical Nyquist-Shannon sampling theorem establishes an upper limit on the bandwidth of signals that can be uniquely reconstructed when sampled at a given sampling rate. To have an analogous result in the realm of graphs, one needs a notion of frequency for graph signals.
Such a spectral interpretation is provided by the eigenvalues and eigenvectors of the Laplacian matrix $\Lcm$, similar to the Fourier transform in traditional signal processing.
The eigenvalues can be thought of as frequencies and indicate the variation in the eigenvectors: a high eigenvalue implies higher variation in the corresponding eigenvector \cite{SPM_Graphs}.
Since the eigenvectors are orthogonal, they form a basis in $\mathbb{R}^N$.
Thus, the \emph{Graph Fourier Transform} (GFT) of a signal $\fv$ is defined as its projection onto the eigenvectors of the graph Laplacian, i.e. $\tilde{\fv}(\lambda_i) = \braket{\fv,\uv_i}$, or more compactly, $\tilde{\fv} = \Um^T\fv$. 

In this context, a smooth or low-pass graph signal can be obtained by forcing high frequency GFT coefficients to vanish. 
More formally, an \emph{$\omega$-bandlimited signal} on a graph is defined to have zero GFT coefficients for frequencies above its bandwidth $\omega$, i.e. its spectral support is restricted to the set of frequencies $[0,\omega]$.
The space of all $\omega$-bandlimited signals is known as the \emph{Paley-Wiener} space and is denoted by $PW_\omega(G)$ \cite{pesenson08}.
Note that $PW_\omega(G)$ is a subspace of $\mathbb{R}^N$.

With the notion of frequency introduced via the GFT, one can frame an adequate \emph{sampling theory} for graph signals using the following ingredients:

\begin{enumerate}[label=\bfseries P\arabic*:, topsep=0pt, itemsep=0pt]
\item \emph{Cutoff frequency -} For a given subset of nodes $\Sc$, find the cut-off frequency $\omega$, such that any $\fv \in PW_{\omega}(G)$ can be exactly recovered from its samples $\fv(\Sc)$.
\item \emph{Optimal sampling set -} For a given cut-off frequency $\omega$, find the the smallest subset of nodes $\Sc_\text{opt}$ (i.e. with minimum $|\Sc_\text{opt}|$) such that all signals $\fv \in PW_\omega(G)$ can be uniquely recovered from their samples $\fv(\Sc_\text{opt})$ on $\Sc_\text{opt}$.
\item \emph{Reconstruction algorithm -} Given samples $\fv(\Sc)$ of a graph signal $\fv$ on a subset of nodes $\Sc$, find the reconstructed signal values $\fv(\Scc)$ on the complementary subset $\Scc$.
\end{enumerate}
Note that for regular sampling in the traditional signal processing, problems \textbf{P1} and \textbf{P2} are reciprocal, i.e., knowing one automatically leads to the solution of the other.
However, this does not hold for irregular sampling, as in the case of graph signals.
Next, we briefly describe the solution to each of the problems above, and refer to \cite{Narang-GlobalSIP-13, Anis-ICASSP-14} for the details. 

\subsection{P1: Cut-off frequency}
\label{sec:P1}

Let $L_2(\Scc)$ denote the space of all graph signals that are zero everywhere except possibly on the nodes in $\Scc$, i.e., $\forall \phi \in L_2(\Scc), \phi(\Sc) = 0$. 
Also, let $\omega(\phi)$ denote the bandwidth of a graph signal $\phi$, i.e., the value of the maximum non-zero frequency of that signal. Then the following theorem can be proved~\cite{Anis-ICASSP-14}:  
\begin{theorem}[Sampling Theorem]
\label{thm:sampling}
For a graph $G$, with normalized Laplacian $\Lcm$, any signal $\fv \in PW_\omega(G)$ can be perfectly recovered from its values on a subset of nodes $\Sc \subset \Vc$ if and only if
\begin{equation}
\omega < \omega_c(\Sc) \defeq \inf_{\phi \in L_2(\Scc)} \omega(\phi)
\end{equation}
where $\omega_c(\Sc)$ is the cut-off frequency.
\end{theorem}
The theorem leads to a cut-off frequency that is {\em lower} than the minimum bandwidth of any signal in $L_2(\Scc)$. 
Intuitively, a signal $\phi \in L_2(\Scc)$ can be added to any input signal $\fv$ without affecting its sampled version (since $\phi$ is identically zero for all vertices that are sampled, i.e., those in $\Sc$). Thus, if there existed a 
$\phi \in L_2(\Scc)$ such that $\phi \in PW_\omega(G)$ we would have that both 
$\fv$ and $\phi + \fv$ belong to $PW_\omega(G)$ {\em and} lead to the same set of samples on $\Sc$. So clearly it would not be possible to recover them both, and thus sampling of such signals in $PW_\omega(G)$ would not be possible. The condition in Theorem~\ref{thm:sampling} ensures that $PW_\omega(G) \cap L_2(\Scc) = \{0\}$ and thus no such $\phi$ exists.  

From Theorem~\ref{thm:sampling}, finding the maximum cut-off frequency for a set $\Sc$ requires finding the bandwidth $\omega(\phi^*)$ of the smoothest possible signal $\phi^* \in L_2(\Scc)$. 
A brute-force approach to this would entail computing the GFT of all signals in $L_2(\Sc)$ and exhaustively searching for $\phi^*$.
We instead devise a computationally efficient way to approximate the bandwidth of any signal $\phi$ for a given integer parameter $k > 0$ as follows:
\begin{equation}
\omega_k(\phi) = \left( \frac{\phi^t \Lcm^k \phi}{\phi^t \phi} \right)^{1/k}
\end{equation}
We then replace $\omega(\phi)$ in Theorem~\ref{thm:sampling} by $\omega_k(\phi)$ in the objective function to obtain our estimated bandwidth:
\begin{equation}
\label{eq:compute_cutoff}
\Omega_k(\Sc) = \inf_{\phi \in L_2(\Scc)} \omega_k(\phi) = \inf_{\phi \in L_2(\Scc)} \left( \frac{\phi^t \Lcm^k \phi}{\phi^t \phi} \right)^{1/k}.
\end{equation}
Then, the smoothest possible signal $\phi^*$ in $L_2(\Scc)$ can be approximated by the minimizer $\phi^*_k$ in~\eqref{eq:compute_cutoff}.
Numerically, $\Omega_k(\Sc)$ and $\phi^*_k$ can be determined from the smallest eigen-pair $(\sigma_{1,k}, \psi_{1,k})$ of the reduced matrix $(\Lcm^k)_\Scc$:
\begin{gather}
\label{eq:cutoff_est}
\Omega_k(\Sc) = \sigma_{1,k}, \\
\label{eq:phi_est}
\phi_k^*(\Scc) = \psi_{1,k}, \; \phi_k^*(\Sc) = \zerov.
\end{gather}
This approach does not require complete eigen-decomposition of $\Lcm$ and is computationally tractable.
One can show that $k$ controls the accuracy of the cut-off estimate (refer to \cite{Anis-ICASSP-14} for details).
As we increase the value of $k$, $\Omega_k(\Sc)$ tends to give a better estimate of the cut-off frequency. Thus, there is a trade-off between accuracy of the estimate on the one hand, and complexity and numerical stability on the other that arise due to the power $k$ in $\Lcm^k$. 
Moreover, $\Omega_k(\Sc)$ can be proven to be always less than the actual cut-off $\omega_c(\Sc)$, i.e. the Sampling Theorem still holds for the subset $\Sc$ except that the class of recoverable signals is determined to be narrower as a penalty for the cut-off approximation.


\subsection{P2: Sampling set}
\label{sec:sampling_set}
We now describe the framework for the converse question: 
given a cut-off frequency $\omega_c$ for $PW_\omega(G)$, what is the smallest sampling set $\Sc_\text{opt}$ so that
a signal $\fv \in PW_\omega(G)$ is uniquely represented by $\fv(\Sc_\text{opt})$. 
If $K_c$ represents the number of eigenvalues of $\Lcm$ below $\omega_c$, then by dimensionality considerations $|\Sc_\text{opt}| \geq K_c$.
Also, note that $\Sc_\text{opt}$ may not be unique.
Formally, one can use Theorem \ref{thm:sampling} and relax the true cut-off $\omega_c(\Sc)$ by $\Omega_k(\Sc)$, then $\Sc_\text{opt}$ can be found from the following optimization problem:
\begin{equation}
\underset{\Sc}{\text{Minimize}} \;\; |\Sc| \quad \text{subject to} \quad \Omega_k(\Sc) \geq \omega_c 
\end{equation}
This is a combinatorial problem because we need to compute $\Omega_k(\Sc)$ for every possible subset $\Sc$.

However, this problem can be solved using a greedy heuristic to get an estimate $\Sc_\text{est}$ of the optimal sampling set.
Starting with an empty sampling set $\Sc$ (with corresponding $\Omega_k(\Sc) = 0$) we keep adding nodes to $\Sc$ (from $\Scc$) one-by-one while trying to ensure maximum increase in $\Omega_k(\Sc)$ at each step.
The hope is that $\Omega_k(\Sc)$ reaches the target cut-off $\omega_c$ with minimum number of node additions to $\Sc$.
To understand which nodes should be included in $\Sc$, we introduce a binary relaxation of our cut-off formulation by defining the following matrix
\begin{equation}
\Mm^\alpha_k(\tv) \defeq \Lcm^k + \alpha \; \Dcm(\tv), \quad k \in \mathbb{Z}^+, \alpha > 0, \tv \in \mathbb{R}^N
\end{equation}
where $\Dcm(\tv)$ is a diagonal matrix
with $\tv$ on its diagonal. 
Let $\left( \lambda^\alpha_k(\tv), \xv^\alpha_k(\tv) \right)$ denote the smallest eigen-pair of $\Mm^\alpha_k(\tv)$.
Then, if $\onev_\Sc: \Vc \rightarrow \{0,1\}$ denotes the indicator function for the subset $\Sc$ (i.e. $\onev(\Sc) = \onev$ and $\onev(\Scc) = \zerov$), one has
\begin{equation}
\label{eq:min_lambda}
\lambda^\alpha_k(\onev_\Sc) 
= \inf_{\xv} \left( \frac{\xv^t \Lcm^k \xv }{\xv^t \xv} + \alpha \frac{\xv(\Sc)^t\xv(\Sc) }{\xv^t \xv} \right)
\end{equation}
Note that the right hand side of the equation above
is simply an \emph{unconstrained regularization} of the constrained optimization problem in (\ref{eq:compute_cutoff}). 
When $\alpha \gg 1$, the components $\xv(\Sc)$ are highly penalized during minimization. Thus, if $\xv^\alpha_k(\onev_\Sc)$ is the minimizer in (\ref{eq:min_lambda}), then $[\xv_\alpha^k(\onev_\Sc)](\Sc) \rightarrow \zerov$, i.e. the values on nodes $\Sc$ tend to be very small. 
Therefore, for $\alpha \gg 1$, we have
\begin{equation}
\label{eq:equiv}
\phi_k^* \approx \xv^\alpha_k(\onev_\Sc),\quad  
\left( \Omega_k(\Sc) \right)^k \approx \lambda^\alpha_k(\onev_\Sc)  
\end{equation}
From the above equation, we observe that the problem of greedily maximizing $\Omega_k(\Sc)$ is equivalent to maximizing $\lambda_k^\alpha(\onev_\Sc)$, and thus, we simply need to study the variation of $\lambda^\alpha_k(\tv)$ with $\tv$, a real-valued vector in $\mathbb{R}^N$, at $\tv = \onev_\Sc$.
This relaxation circumvents the combinatorial nature of our problem and has been used earlier to study graph partitioning based on Dirichlet eigenvalues~\cite{Osting-arxiv-13}.
The gradient of $\lambda^\alpha_k(\tv)$ with respect to $\tv(i)$ is given by
\begin{equation}
\left. \frac{d\lambda_\alpha^k(\tv)}{d\tv(i)} \right|_{\tv = \onev_\Sc} = \alpha \left([\xv_\alpha^k(\onev_\Sc)](i) \right)^2 \approx \alpha (\phiv_k^*(i))^2.
\end{equation}
This equation forms the basis of our greedy heuristic: starting with an empty $\Sc$ (i.e., $\onev_\Sc = \zerov$), if at each step, we include the node on which the smoothest signal $\phi_k^* \in L_2(\Scc)$ has maximum energy (i.e., $\onev_\Sc(i) \leftarrow 1, i = \text{arg max}_j\left[(\phiv_k^*(j))^2\right]$), then the cut-off estimate $\Omega_k(\Sc)$ tends to increase maximally. 

While the algorithm in \cite{Anis-ICASSP-14} has a goal of finding an $\Sc$ of smallest possible size that satisfies a target cut-off frequency, we can easily adapt it for our cut-off frequency maximization-based active learning algorithm. This will be discussed in detail in Section~\ref{sec:proposed}.

\subsection{P3: Reconstruction}
\label{sec:recon}
A graph signal $\fv \in PW_{\omega}(G)$ can be written as a linear combination eigenvectors of $\Lcb$ with eigenvalues less than $\omega$, i.e., $\fv = \Um_{\Vc,\Kc} \boldsymbol{\alpha}$ where $\Kc$ is the index set of those eigenvectors and $\boldsymbol{\alpha}$ is a vector containing the corresponding GFT coefficients. 
When the unique recovery conditions of Theorem~\ref{thm:sampling} are satisfied, $\boldsymbol{\alpha}$ and the signal $\fv$, can be recovered from its subsampled version $\fv(\Sc)$ by solving the following least squares problem:
\begin{align}
\fv(\Sc) &= \Um_{\Sc,\Kc} \boldsymbol{\alpha} \\
\Rightarrow 
\boldsymbol{\alpha} &= \Um_{\Sc,\Kc}^{+} \fv(\Sc).
\end{align}
%
%
Note that if the original signal $\fv$ is not bandlimited, i.e., $\fv \notin PW_\omega(G)$, then the least squares solution corresponds to an approximation of $\fv$ in $PW_\omega(G)$ (in $l_2$ sense). 

The least squares solution requires eigen-decomposition of $\Lcb$ which is computationally expensive and may not be practical for large graphs. We now describe the iterative, distributed algorithm developed in \cite{Narang-GlobalSIP-13} based on projection onto convex sets~(POCS). 
The proposed method is similar to the Papoulis-Gerchberg algorithm~\cite{Sauer-87} in classical signal processing which is used to reconstruct a bandlimited signal from irregular samples. The convex sets of interest in this case are
\begin{align}
\label{eq:c1}
C_1 &= \{\xv: \Ds\xv = \Ds\fv \} \\
\label{eq:c2}
C_2 &= PW_{\omega}(G),
\end{align}
where $\Ds$ is the downsampling operator such that $\Ds\fv = \fv(\Sc)$.
The unique solution $\fv$ to the least squares problem satisfies the following two constraints:
(1)~the signal equals the known values on the sampling set (i.e., $\fv \in C_1$),
(2)~the signal is $\omega$-bandlimited, where $\omega$ is computed using~(\ref{eq:cutoff_est}) (i.e., $\fv \in C_2$).
The projector for $C_2$ is $\boldsymbol{\Pc}_{\omega}:\mathbb{R}^N \rightarrow PW_{\omega}(G)$ which is a low-pass graph filter such that
\begin{equation}
\boldsymbol{\Pc}_{\omega} \xv \in PW_{\omega}(G) \quad \forall \; \xv \in \mathbb{R}^N
\end{equation}
$\boldsymbol{\Pc}_{\omega}$ can be written in graph spectral domain as $\boldsymbol{\Pc}_{\omega} = \Hm(\Lcb) = \sum_{i = 1}^N h(\lambda_i) \uv_i \uv_i^t$ where
\begin{equation}
h(\lambda) = \left\{ 
  \begin{array}{l l}
    1, & \quad \text{if $\lambda < \omega$}\\
    0, & \quad \text{if $\lambda \geq \omega$}
  \end{array} \right.
  \label{eq:ideal_kernel}
\end{equation}
We define the projection operator for $C_1$ as $\boldsymbol{\Pc}_{\Sc}:\mathbb{R}^N \rightarrow C_1$ which replaces the samples on $\Sc$ by the known values.
\begin{equation}
\boldsymbol{\Pc}_{\Sc}\xv = \xv + \Ds^t(\fv(\Sc) - \Ds \xv).
\end{equation}
\begin{figure}
\centering
  \includegraphics[width=0.2\textwidth]{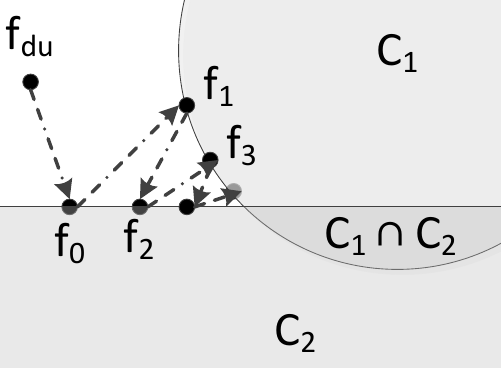}
  \caption{Iterative reconstruction using POCS}
  \label{fig:pocs}
\end{figure}
With this notation the proposed iterative algorithm  can be written as:
\begin{align}
\fv_0 &= \boldsymbol{\Pc}_{\omega} (\Ds^t \fv(\Sc)) \nonumber\\
\fv_{i+1} &= \boldsymbol{\Pc}_{\omega} \boldsymbol{\Pc}_{\Sc} \fv_i
\label{eq:iter_band}
\end{align}
At each iteration the algorithm resets the signal samples on $\Sc$ to the actual given samples and then projects the signal onto the low-pass space $PW_{\omega}(G)$.  
Figure~\ref{fig:pocs} depicts this procedure graphically. It can be shown that $\boldsymbol{\Tc} = \boldsymbol{\Pc}_{\omega}\boldsymbol{\Pc}_{\Sc}$ is a non-expansive and asymptotically regular operator. Hence, the iterations in \eqref{eq:iter_band} converge to the unique point $\fv \in C_1 \cap C_2$ which is the desired solution.

The low pass filter $\boldsymbol{\Pc}_{\omega}$ above is a spectral graph filter with an ideal  brick-wall type spectral response. Thus, the exact computation of $\boldsymbol{\Pc}_{\omega}$ would require knowledge of the GFT, which we would like to avoid due to high computational complexity for large graphs. 
However, it is possible to approximate the ideal filtering operation as a matrix polynomial in terms of $\Lcb$, that can be implemented efficiently using only matrix vector products. Thus we replace $\boldsymbol{\Pc}_{\omega}$ in~(\ref{eq:iter_band}) with an approximate low pass filter $\boldsymbol{\Pc}_\omega^{\text{poly}}$ given by:
\begin{equation}
\boldsymbol{\Pc}_\omega^{\text{poly}} = \sum_{i = 1}^N  \left(\sum_{j = 0}^p a_j \lambda_i ^j \right) \uv_i \uv_i^t = \sum_{j = 0}^p a_j \Lcb ^j
\end{equation}
We specifically use the truncated Chebychev polynomial expansion of any spectral kernel $h(\lambda)$, as proposed in~\cite{hammond11}, in our experiments. It is easy to show that an operator which is a $p$-degree polynomial in $\Lcb$ is $p$-hop localized on the graph and can be implemented in a distributed fashion. In order to ensure that the Chebyshev polynomial approximation is good, we first approximate the ideal spectral kernel by a smooth, continuous sigmoid-like function~(see Figure~\ref{fig:spectral_response})
\begin{equation}
h'(\lambda) = \frac{1}{\left(1+ \exp(\alpha(\lambda - \omega))\right)}
\end{equation}
Due to these approximations in the filter, the reconstructed signal obtained via POCS is different from the true band-limited signal. However, in semi-supervised learning applications we do not expect the signals (i.e., class membership functions) to be exactly bandlimited anyway. So using a filter with slowly decaying spectral response ends up improving the classification accuracy slightly. 
\begin{figure}
\centering
\includegraphics[width = 0.30\textwidth]{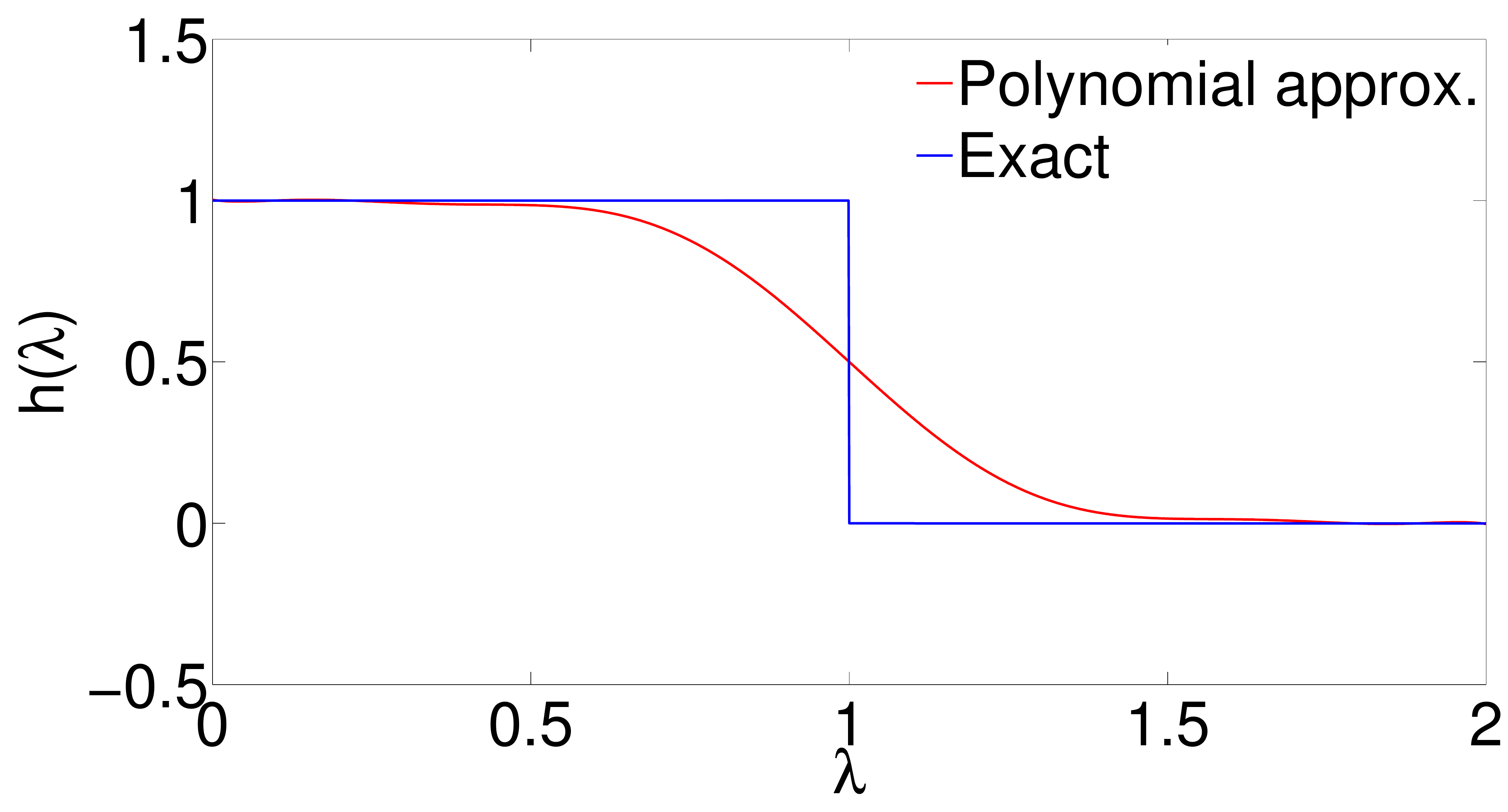}
\caption{Spectral response of an approximate polynomial filter of degree 10. $\omega = 1, \alpha = 8$.}
\label{fig:spectral_response}
\end{figure}
%


\section{Graph Sampling Based Active \\ semi-supervised learning}
\label{sec:active_ssl}
We now relate the sampling theory developed for graph signals to active semi-supervised learning and propose our solution to the problem.
As noted earlier, if the edges of the graph represent similarity between the nodes, then a graph signal defined using the membership functions of a particular class tends to be smooth. 
This is illustrated experimentally in Figure~\ref{fig:gft}.
In Section~\ref{sec:P1} we showed how to estimate the sampling cut-off frequency for a set of vertices. 
In practice, class membership signals are not strictly bandlimited (see Figure~\ref{fig:gft}). Thus we will be approximating a non-bandlimited signal with one that is bandlimited to the cut-off frequency of the chosen vertex set. 
The key observation in our work is that, even though we cannot recover the ``true'' membership signal exactly from its samples, an active learning approach {\em should aim at selecting the sampling set with maximum cut-off frequency}. This is obviously true since $PW_\omega(G) \subset PW_{\omega'}(G)$ if $\omega \leq \omega'$ and thus, for any signal, its best approximation with a signal from $PW_{\omega'}(G)$ can be no worse (in terms of $l_2$ error) than its best approximation with a signal from $PW_\omega(G)$. 

In this setting, predicting the labels of the unknown data-points using the labeled data amounts to reconstructing a bandlimited graph signal from its values on the sampling set.
Thus, based on the above reasoning the active learning strategy,  
given a target number of datapoints to be labeled, should be to find a set $\Sc$, with that size, so that the  cut-off frequency of $\Sc$ is maximized. 


%
\begin{figure*}
\centering
\begin{subfigure}{0.32\textwidth}
  \centering
  \includegraphics[width=0.9\textwidth]{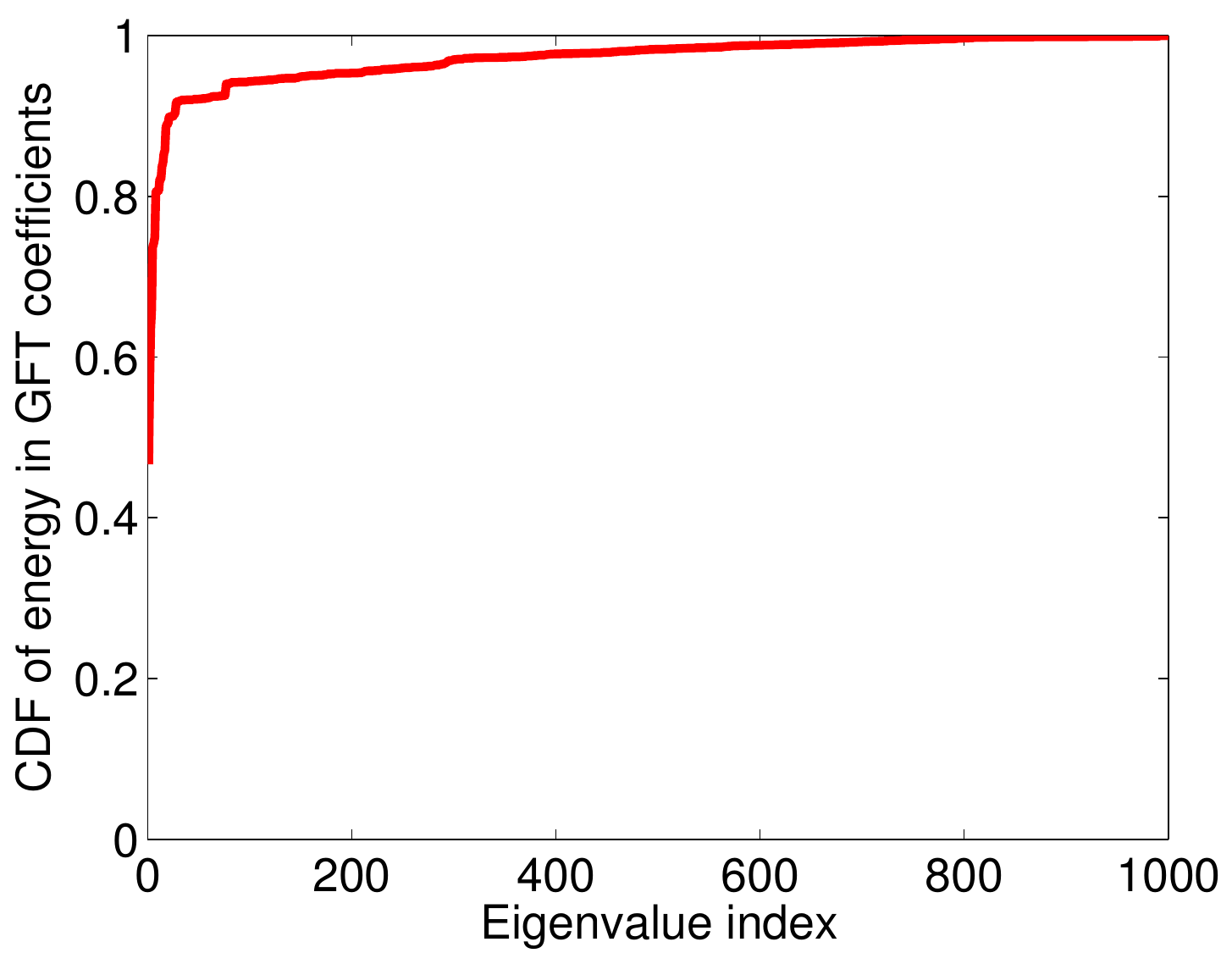}
  \caption{USPS}
  \label{fig:gft_usps}
\end{subfigure}
\begin{subfigure}{0.32\textwidth}
  \centering
  \includegraphics[width=0.9\textwidth]{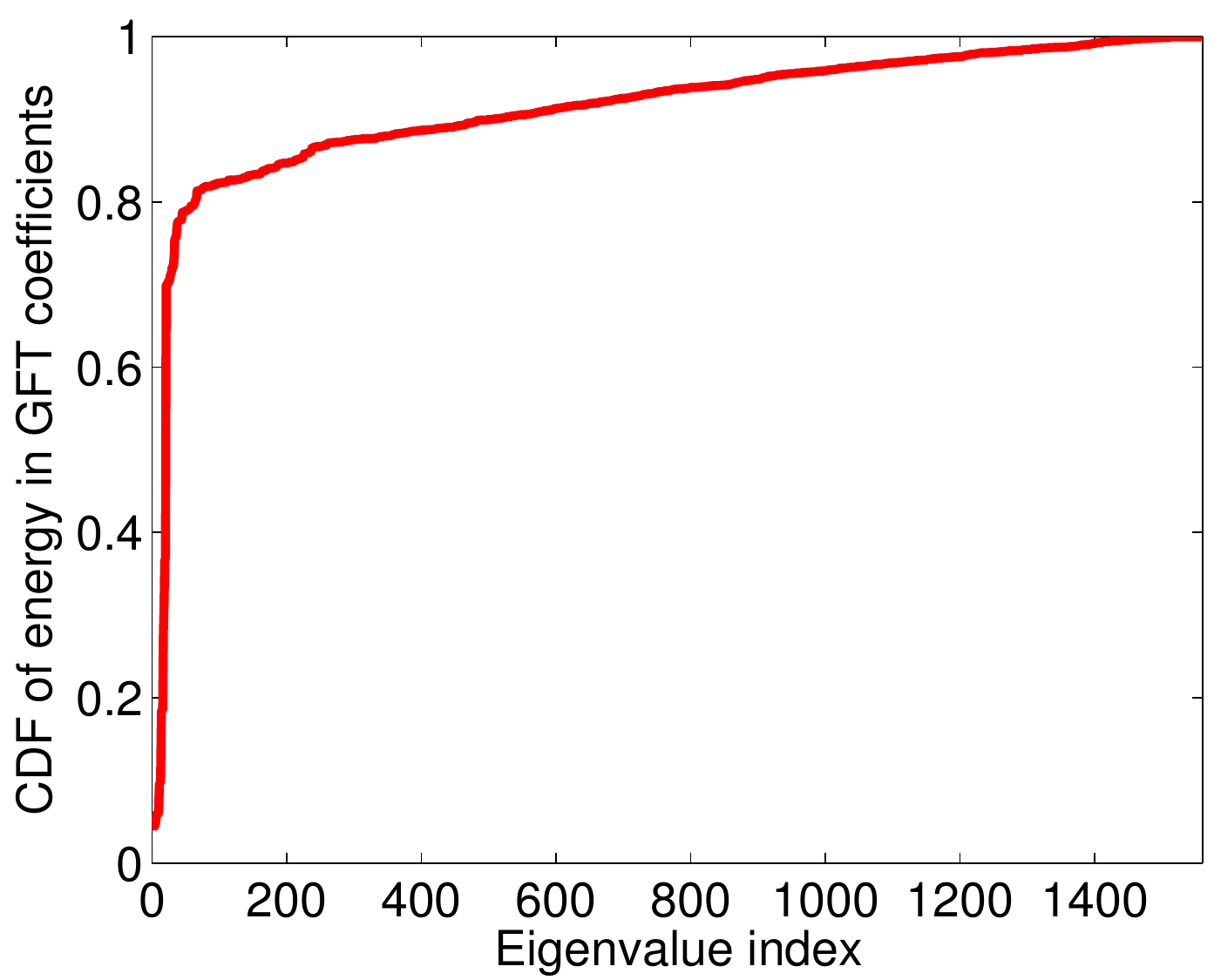}
  \caption{Isolet}
  \label{fig:gft_isolet}
\end{subfigure}
\begin{subfigure}{0.32\textwidth}
  \centering
  \includegraphics[width=0.9\textwidth]{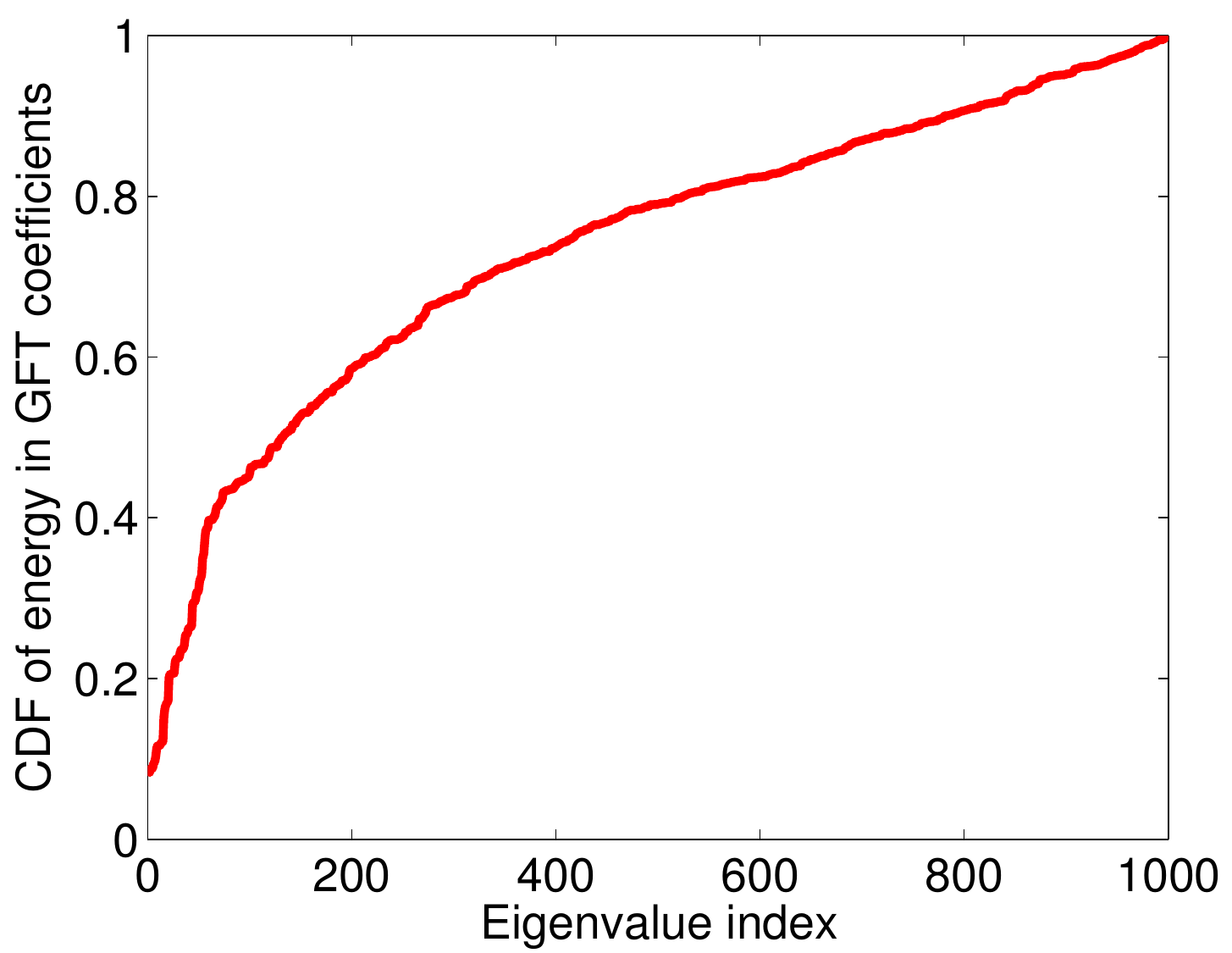}
  \caption{20 newsgroups}
  \label{fig:gft_newsgroups}
\end{subfigure}
\caption{Cumulative distribution of energy in the GFT coefficients of one of the class membership functions pertaining to the three real-world dataset experiments considered in Section~\ref{sec:experiments}. Note that most of the energy is concentrated in the low-pass region.}
\label{fig:gft}
\end{figure*}
%


\subsection{Proposed method}
\label{sec:proposed}
Now, we present the details of our method. We target a multi-class active semi-supervised learning problem with $C$ classes. The true membership function for class $j$ is denoted as $\fv_j: \Vc \mapsto \{0,1\}$, where $\fv_j(i) = 1$ indicates that node $i$ belongs to class $j$. 
These membership functions are taken to be the graph signals for our setting. 
The predicted membership functions for each class take real values and are denoted as $\hat{\fv}_j:\Vc \mapsto \mathbb{R}$.
The predicted label of node $i$ is given by $\argmax_{j} \hat{\fv}_j(i)$. 
We denote the labeled set as $\Sc_L$ and the unlabeled set as $\Sc_U = \Vc \setminus \Sc_L$. 
Then, our solution to the active semi-supervised learning task can be formally summarized as follows:
\begin{enumerate}[itemsep=0pt]
\item Given a size $m$ and parameter $k$, we first find the optimal labeled set $\Sc_L^*$ and corresponding cut-off frequency $\Omega_k(\Sc_L^*)$ as follows:
\begin{equation}
\label{eq:max_prob}
\Sc_L^* = \argmax_{\Sc:|\Sc| = m} \; \Omega_k(\Sc) 
\end{equation}
We solve this problem in a greedy fashion by adding nodes to $\Sc$ that maximize the increase in $\Omega_k(\Sc)$ at each step (cf. Section~\ref{sec:sampling_set}). This procedure is summarized with Algorithm~\ref{alg:alg1}
\item Next, we query the labels of nodes in $\Sc_L^*$.
\item Finally, we determine the predicted membership functions $\hat{\fv}_j$ for each class from $\fv_j(\Sc_L^*), j = 1,\dots,C$ using the POCS iterative method described in Section~\ref{sec:recon}, where $\Sc = \Sc_L^*$ and $\omega = \Omega_k(\Sc_L^*)$ are used in \eqref{eq:c1} and \eqref{eq:c2} to construct the convex sets.
\end{enumerate}
\begin{algorithm}[t]
\caption{Greedy heuristic for finding $\Sc_L^*$}
\label{alg:alg1}
\begin{algorithmic}[1]
\REQUIRE $G = \{\Vc,E\}$, $\Lcm$, target size $m$, parameter $k \in \mathbb{Z}^+$.
\ENSURE $\Sc = \{ \emptyset \}$.
\WHILE{$|\Sc| \leq m$}
\STATE For $\Sc$, compute the smoothest signal $\phi_k^* \in L_2(\Scc)$ using (\ref{eq:cutoff_est}) and (\ref{eq:phi_est}).
\STATE $v \leftarrow \text{arg max}_i \left[(\phi_k^*(i))^2\right]$.
\STATE $\Sc \leftarrow \Sc \cup v$.
\ENDWHILE
\STATE $\Sc_L^* \leftarrow \Sc$.
 \end{algorithmic}
\end{algorithm}
\subsection{Graph Theoretic Interpretation}
\label{sec:graph_interpretation}
In this section, we will provide an intuitive interpretation for our node selection algorithm in terms of connected-ness among the nodes. To simplify the exposition, we consider the maximization problem~\eqref{eq:max_prob} for $k = 1$:
\begin{equation}
\Omega_1(\Sc) = \inf_{\substack{\xv(\Sc) = \zerov \\ ||\xv|| = 1}} \xv^t \Lcm \xv
\end{equation}
This expression appears more commonly as part of discrete Dirichlet eigenvalue problems on graphs. Specifically, it is equal to the Dirichlet energy of the subset $\Scc$ \cite{Chung-97, Osting-arxiv-13}.
The sampling set selection problem seeks to identify the subset $\Sc$ that maximizes this objective function.
To give an intuitive interpretation of our goal, we expand the objective function for any $\xv$ with constraint $\xv(\Sc) = \zerov$ as follows:
\begin{align}
\xv^t \Lcm \xv &= \sum_{i \sim j} w_{ij} \left(\frac{x_i}{\sqrt{d_i}} - \frac{x_j}{\sqrt{d_j}} \right)^2 \nonumber \\
&= \sum_{\substack{i\sim j\\ i \in \Sc, j \in \Scc}} w_{ij} \left( \frac{x_j^2}{d_j} \right)
+ \sum_{\substack{ i \sim j \\ i,j \in \Scc}} w_{ij} \left(\frac{x_i}{\sqrt{d_i}} - \frac{x_j}{\sqrt{d_j}} \right)^2.   \nonumber \\
\end{align}
The minimizer in the equation above is the first Dirichlet eigenvector which is guaranteed to have strictly positive values on $\Scc$ \cite{Osting-arxiv-13}. 
Therefore, the contribution of the second term is expected to be negligible as compared to the first one due to differencing, and we get 
\begin{equation}
\xv^t \Lcm \xv \approx \sum_{j \in \Scc} \left( \frac{p_j}{d_j} \right) x_j^2,
\end{equation}
where, $p_j = \sum_{i \in \Sc} w_{ij}$ is defined as the ``partial out-degree" of node $j \in \Scc$, i.e., it is the sum of weights of edges crossing over to the set $\Sc$.
Therefore, given a current selected $\Sc$, the greedy algorithm selects the next node, to be added to $\Sc$, that maximizes the increase in
\begin{equation}
\Omega_1(\Sc) \approx \inf_{||\xv||=1} \; \sum_{j \in \Scc} \left( \frac{p_j}{d_j} \right) x_j^2.
\end{equation}
Due to the constraint $||\xv||=1$, the expression being minimized is essentially an infimum over a convex combination of the fractional out-degrees and its value is largely determined by nodes $j \in \Scc$ for which $p_j/d_j$ is small. 
In other words, we must worry about those nodes that have a low ratio of partial degree to the actual degree.
Thus, in the simplest case, our selection algorithm tries to remove those nodes from the unlabeled set that are weakly connected to nodes in the labeled set.
This makes intuitive sense as, in the end, most prediction algorithms involve propagation of labels from the labeled to the unlabeled nodes. If an unlabeled node is strongly connected to various numerous points, its label can be assigned with greater confidence.

Note that using a higher power $k$ in the cost function, i.e., finding $\Omega_k(\Sc)$ for $k>1$ involves $\xv \Lcm^k \xv$ which, loosely speaking, takes into account higher order interactions between the nodes while choosing the nodes to label. In a sense, we expect it to capture the connectivities in a more \emph{global} sense, beyond local interactions, taking into account the underlying manifold structure of the data. 

\subsection{Complexity}
We now comment on the time and space complexity of our algorithm. The most complex step in the greedy procedure for maximizing $\Omega_k(\Sc)$ is computing the smallest eigen-pair of $(\Lcm^k)_\Scc$. 
This can be accomplished using an iterative Rayleigh-quotient minimization based algorithm. 
Specifically, the locally-optimal pre-conditioned conjugate gradient (LOPCG) method \cite{knyazev-SIAM-01} is suitable for this approach. 
Note that  $(\Lcm^k)_\Scc$  can be written as $\mathbf{I}_{\Scc,\Vc} . \Lcm . \Lcm \dots \Lcm. \mathbf{I}_{\Vc, \Scc}$, hence the eigenvalue computation can be broken into atomic matrix-vector products: $\Lcm.\xv$.
Typically, the graphs encountered in learning applications are sparse, leading to efficient implementations of $\Lcm.\xv$.
If $|\Lcm|$ denotes the number of non-zero elements in $\Lcm$, then the complexity of the matrix-vector product is $O(|\Lcm|)$.
The complexity of each eigen-pair computation for $(\Lcm^k)_\Scc$ is then $O(k|\Lcm|r)$, where $r$ is a constant equal to the average number of iterations required for the LOPCG algorithm ($r$ depends on the spectral properties of $\Lcm$ and is independent of its size $|\Vc|$).
The complexity of the label selection algorithm then becomes $O(k|\Lcm| mr)$, where $m$ is the number of labels requested.

In the iterative reconstruction algorithm, since we use polynomial graph filters (Section~\ref{sec:recon}), once again the atomic step is the matrix-vector product $\Lcm.\xv$.
The complexity of this algorithm can be given as $O(|\Lcm|pq)$, where $p$ is the order of the polynomial used to design the filter and $q$ is the average number of iterations required for convergence. Again, both these parameters are independent of $|\Vc|$.
Thus, the overall complexity of our algorithm is $O(|\Lcm|(kmr+pq))$.
In addition, our algorithm has major advantages in terms of space complexity: 
Since, the atomic operation at each step is the matrix-vector product $\Lcm.\xv$, we only need to store $\Lcm$ and a constant number of vectors. Moreover, the structure of the Laplacian matrix allows one to perform the aforementioned operations in a distributed fashion. This makes it well-suited for large-scale implementations using software packages such as GraphLab \cite{Graphlab}.

\subsection{Prediction Error and Number of Labels}
\label{sec:pred_error}

As discussed in Section~\ref{sec:recon}, given the samples $\fv_{\Sc}$ of the true graph signal on a subset of nodes $\Sc \subset \Vc$, its estimate on $\Scc$ is obtained by solving the following problem:
\begin{equation}
\hat{\fv}(\Scc) = \Um_{\Scc,\Kc}\boldsymbol{\alpha}^* \text{ where, }
\mbb{\alpha}^* = \argmin_{\mbb{\alpha}} \|\Um_{\Sc,\Kc} \mbb{\alpha} - \fv(\Sc)\|
\label{eq:ls_recon}
\end{equation}
Here, $\Kc$ is the index set of eigenvectors with eigenvalues less than the cut-off $\omega_c(\Sc)$. If the true signal $\fv \in PW_{\omega_c(\Sc)}(G)$, then the prediction is perfect. However, this is not the case in most problems. The prediction error $\|\fv-\hat{\fv}\|$ roughly equals the portion of energy of the true signal in $[\omega_c(\Sc), \lambda_N]$ frequency band. By choosing the sampling set $\Sc$ that maximizes $\omega_c(\Sc)$, we try to capture most of the signal energy and thus, reduce the prediction error. 



%
%
%
%
%
%

%

An important question in the context of active learning is determining the minimum number of labels required so that the prediction error $\|\fv - \hat{\fv}\|$ is less that some given tolerance $\delta$. To find this we first characterize the smoothness $\gamma(\fv)$ of a signal $\fv$ as
\[\gamma(\fv) = \min \theta \; \text{ s.t. } 
\|\fv - \mbb{\Pc}_{\theta}\fv\| \leq \delta\]
The following theorem gives a lower bound on the minimum of number of labels required in terms of $\gamma(\fv)$.
\begin{theorem}
If $\hat{\fv}$ is obtained by solving~\eqref{eq:ls_recon}, then the minimum number of labels $l$ required to satisfy $\|\fv - \hat{\fv}\| \leq \delta$ is greater than $p$, where $p$ is the number of eigenvalues of $\Lcb$ less than $\gamma(\fv)$.
\end{theorem}
\begin{proof}
In order for \eqref{eq:ls_recon} to have a unique solution, $\Um_{\Sc,\Kc}$ needs to have full column rank, which implies that
$l = |\Sc| \geq |\Kc|$. Now, for $\|\fv - \hat{\fv}\| \leq \delta$ to hold the bandwidth of $\hat{\fv}$ has to be at least $\gamma(\fv)$, or in other words, $|\Kc| \geq p$. This gives us the desired result as $l \geq |\Kc| \geq p$.   
\end{proof}
%


\section{Related Work}
\label{sec:related_work}
Different frameworks have been proposed for pool-based batch-mode active semi-supervised learning including optimal experiment design~\cite{Zhang-PAMI-11, Yu-ICML-06}, generalization error bound minimization~\cite{Gu-ICDM-12, Gu-NIPS-12} and submodular optimization~\cite{Guillory-NIPS-09, Guillory-UAI-11, Hoi-ICML-06}. We now point out connections between some of the graph based approaches in the above categories and our graph signal sampling theory based framework.

The notion of frequency given by GFT is closely related to Laplacian eigenmaps which is a well known dimensionality reduction technique~\cite{Belkin-ML-04}. GFT can be viewed as a way of measuring the signal variation on the manifold represented by Laplacian eigenmaps. By selecting nodes that maximize the bandwidth of the space of recoverable signals, we are  trying to capture as many dimensions of the manifold structure of the data with as few samples as possible. A related active learning method proposed by Zhang et al.~\cite{Zhang-PAMI-11} uses optimal experiment design while considering local structure of the data in a way which is similar to local linear embedding (LLE) for approximating the underlying low-dimensional manifold~\cite{Roweis-SCI-00}. This approach tries to choose the most representative data points from which one can recover the whole data set by local linear reconstruction. It is interesting to note that under certain conditions LLE and Laplacian eigenmaps are equivalent~\cite{Kong-ICML-12}. 

%

%

Gu and Han~\cite{Gu-ICDM-12} propose a method based on minimizing the generalization error bound for learning with local and global consistency (LLGC)~\cite{Zhou-NIPS-04}. Their formulation boils down to choosing subset $S$ that minimizes $\Tr\left((\mu \Lm_{S} + \mathbf{I})^{-2}\right)$. To relate this formulation to our proposed method, note that
\[\Tr\left((\mu \Lm_{S} + \mathbf{I})^{-2}\right) = \sum_{i} \frac{1}{(\zeta_i + 1)^2} \leq \frac{|S|}{(\zeta_1 + 1)^2} \]
%
where, $\zeta_1 \leq \ldots \leq \zeta_{|S|}$ denote the eigenvalues of $\Lm_S$. Loosely speaking, minimizing the above objective function is equivalent to maximizing the smallest eigenvalue $\zeta_1$ of $\Lm_S$. So, this method essentially tries to ensure that the labeled set is well-connected to the unlabeled set whereas our method ensures that the unlabeled set is well-connected to the labeled set~(cf. Section~\ref{sec:graph_interpretation}). 

Submodular functions have been used for active semi-supervised learning on graphs by Guillory and Bilmes~\cite{Guillory-UAI-11, Guillory-NIPS-09}. In this work, the subset of nodes $S\subset \Vc$ is chosen to maximize
\begin{equation}
\Psi(S) = \min_{T \subseteq \Vc \setminus S : T \neq \emptyset} \frac{\Gamma(T)}{|T|},
\end{equation}
where $\Gamma(T)$ denotes the cut function $\sum_{i\in T, j \notin T} w_{ij}$. Intuitively, maximizing $\Psi(S)$ ensures that no subset of unlabeled nodes is weakly connected to the labeled set $S$. This agrees with the graph theoretic interpretation of our method given in Section~\ref{sec:graph_interpretation}. They also provide a bound on the prediction error in terms $\Psi(S)$ 
and a smoothness function $\Phi(\fv) = \sum_{i,j}w_{ij}|f_i-f_j|$.
%
%
This bound gives a theoretical justification for semi-supervised learning using min-cuts~\cite{Blum-ICML-01}. 
It also motivates a graph partitioning-based active learning heuristic \cite{Guillory-NIPS-09}  
which says that to select $l$ nodes to label, the graph should be partitioned into $l$ clusters and one node should be picked at random from each cluster.


\section{Experiments}
\label{sec:experiments}
We compare our method against three active semi-supervised learning approaches mentioned in the previous section, namely, LLR~\cite{Zhang-PAMI-11}, LLGC error bound minimization~\cite{Gu-ICDM-12}, METIS graph partitioning based heuristic~\cite{Guillory-NIPS-09} and $\Psi$-max~\cite{Guillory-UAI-11}. The details of implementation of each method are as follows:
\begin{enumerate}
\item The LLR approach~\cite{Zhang-PAMI-11} allows any prediction method once the samples to be queried are chosen. We use the Laplacian regularized least squares (LapRLS)~\cite{Belkin-JMLR-06} method for prediction (used in~\cite{Zhang-PAMI-11}). 
\item In our implementation of the LLGC bound method~\cite{Gu-ICDM-12}, we fix the parameter $\mu$ to 0.01. Since this approach is based on minimizing the generalization error bound for LLGC, we use the same method for prediction with the queried samples.
\footnote{In our experiments, we observed that the greedy algorithm given in \cite{Gu-ICDM-12} did not converge to a good solution. So we use Monte-Carlo simulations to minimize the objective function.}
\item The normalized cut based active learning heuristic of Guillory and Bilmes~\cite{Guillory-NIPS-09} is implemented using the METIS graph partitioning package~\cite{Karypis-SIAM-98}. This algorithm chooses a random node to label from each partition, so we average the error rates over a 100 trials.
\end{enumerate}
In the implementation of our proposed method, we use approximate polynomial filters of degree 10 with $\alpha = 8$. The parameter $k$ in our method is fixed as 8 for these experiments. Its effect on classification accuracy is studied in Section~\ref{sec:k_effect}. In addition to the above methods, we also compare with the random sampling strategy. We use LapRLS to predict the unknown labels from the randomly queried samples and report the average error rates over 30 trials. 

To intuitively demonstrate the effectiveness of our method, we first test it on a two circles toy data as shown in Figure~\ref{fig:toy}. The data is comprised of 200 nodes from which we would like to select 8 nodes to query. We construct a weighted sparse graph by connecting each node to its 10 nearest neighbors while ensuring that the connections are symmetric. The edge weights are computed with the Gaussian kernel $\exp\left(- \frac{||\xv_i-\xv_j||^2}{2\sigma^2} \right)$ (except in the case of $\Psi$-max where the graph is unweighted). It can be seen from Figure~\ref{fig:toy} that all the methods choose 4 points from each of the two circles. Additionally, the proposed approach selects evenly spaced data points within one circle, while at the same time maximizing the spacing between the selected data points in different circles. This is in accordance with the requirement of choosing points which are most representative of the data.
\begin{figure*}
\centering
\begin{subfigure}{0.24\textwidth}
  \centering
  \includegraphics[width=0.9\textwidth]{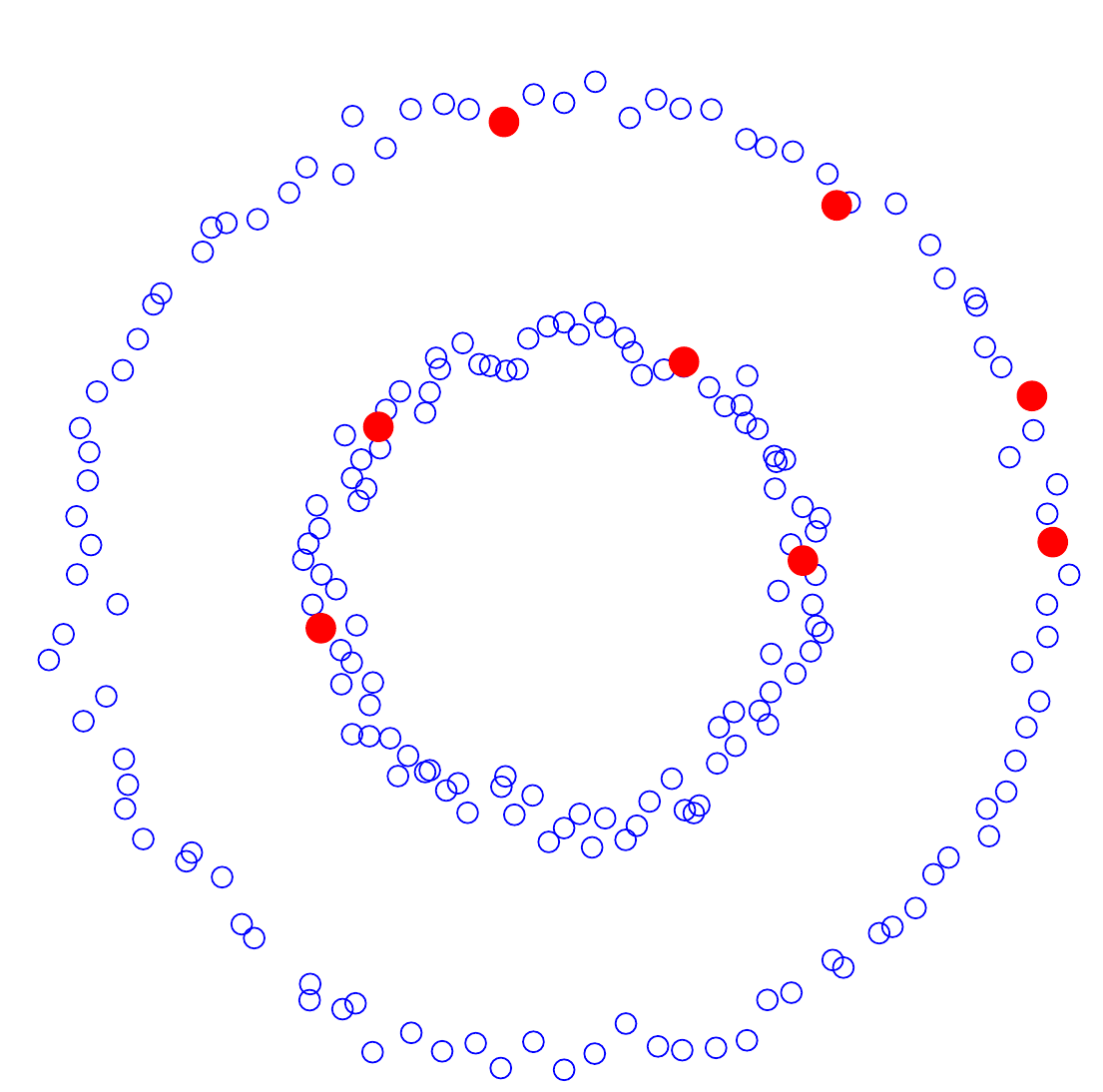}
  \caption{$\Psi$-max}
  \label{fig:Guillory_UAI_toy}
\end{subfigure}
\begin{subfigure}{0.24\textwidth}
  \centering
  \includegraphics[width=0.9\textwidth]{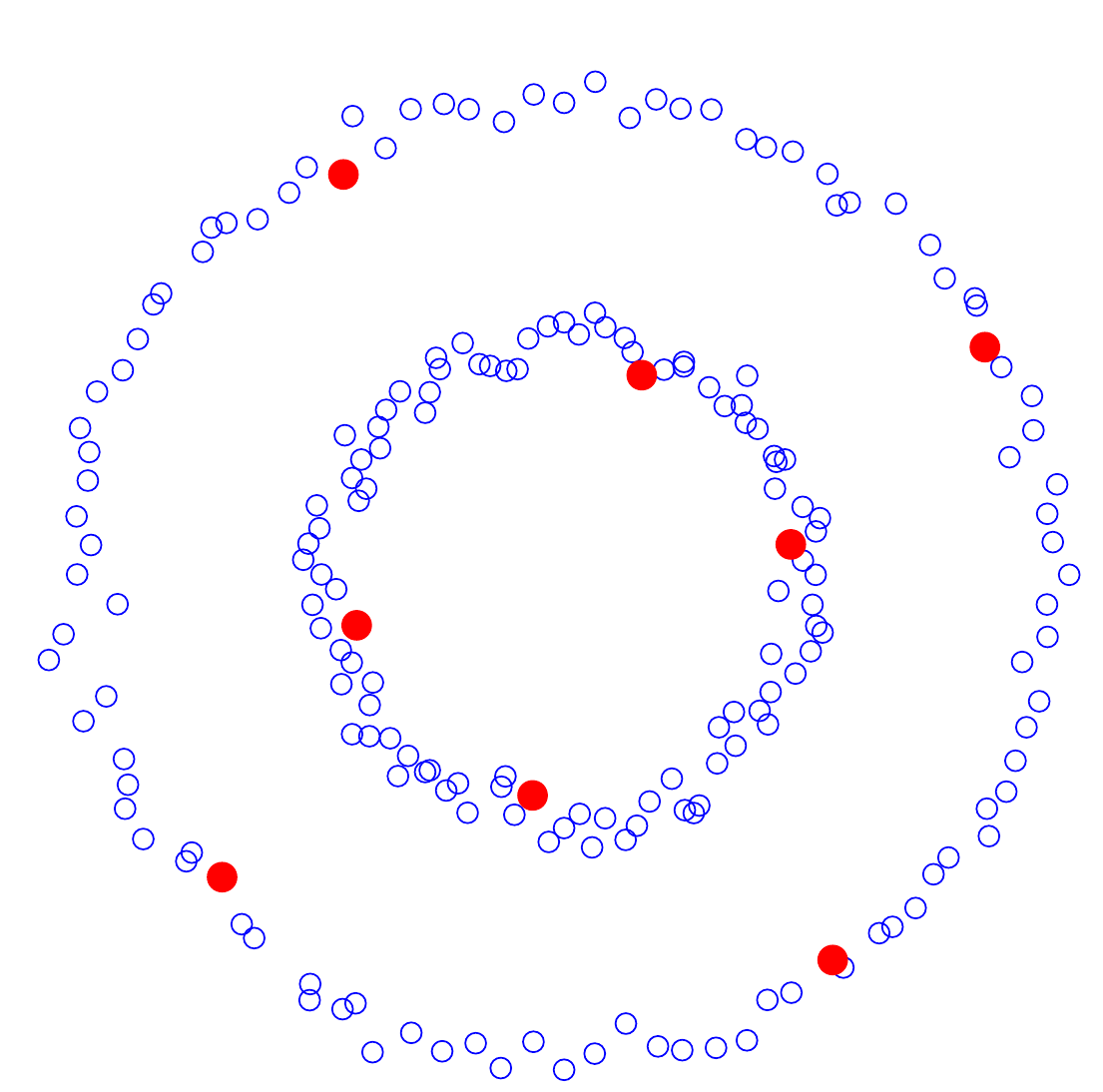}
  \caption{LLR}
  \label{fig:Zhang_PAMI_toy}
\end{subfigure}
\begin{subfigure}{0.24\textwidth}
  \centering
  \includegraphics[width=0.9\textwidth]{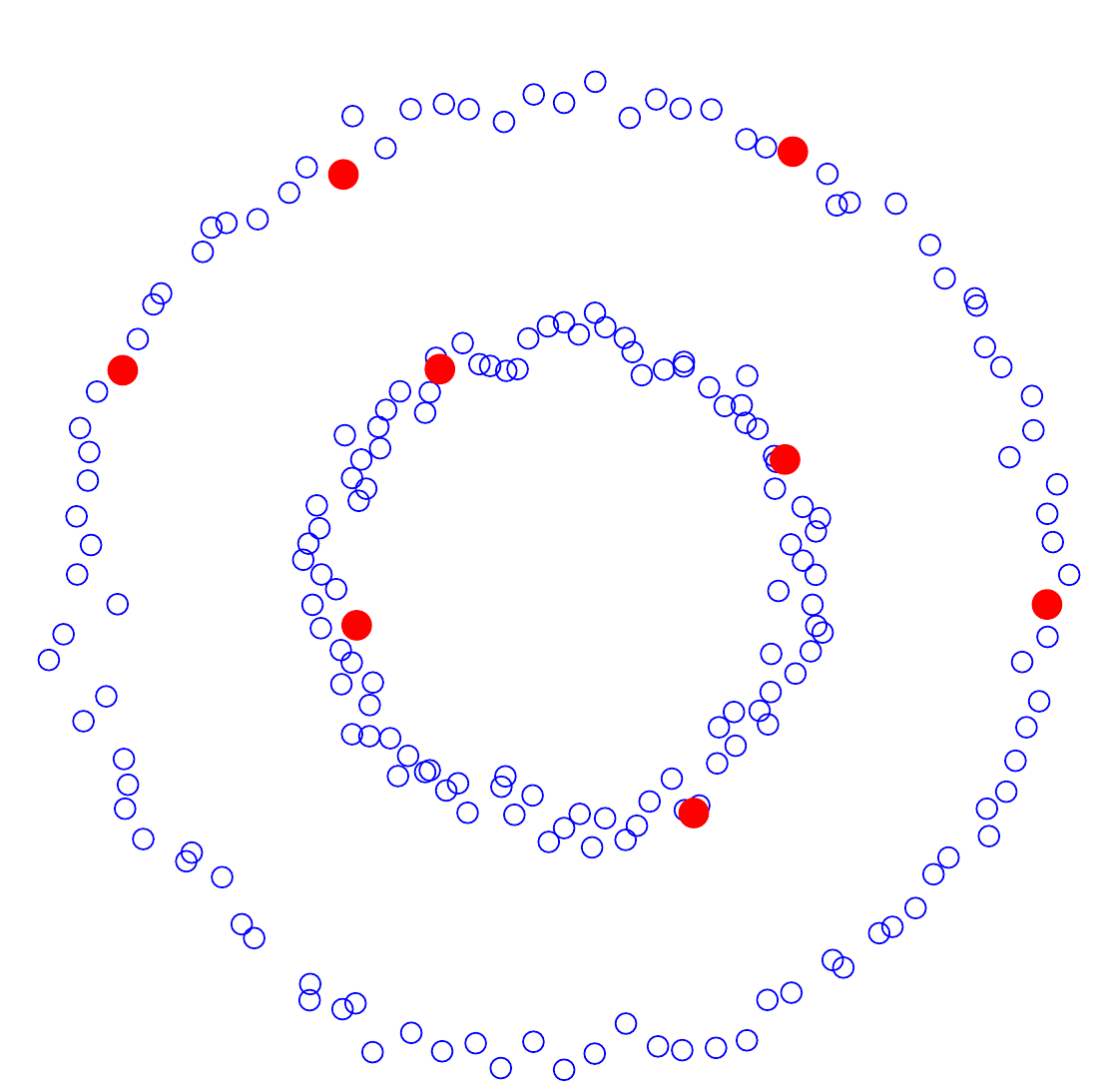}
  \caption{LLGC Bound}
  \label{fig:Gu_ICDM_toy}
\end{subfigure}
\begin{subfigure}{0.24\textwidth}
  \centering
  \includegraphics[width=0.9\textwidth]{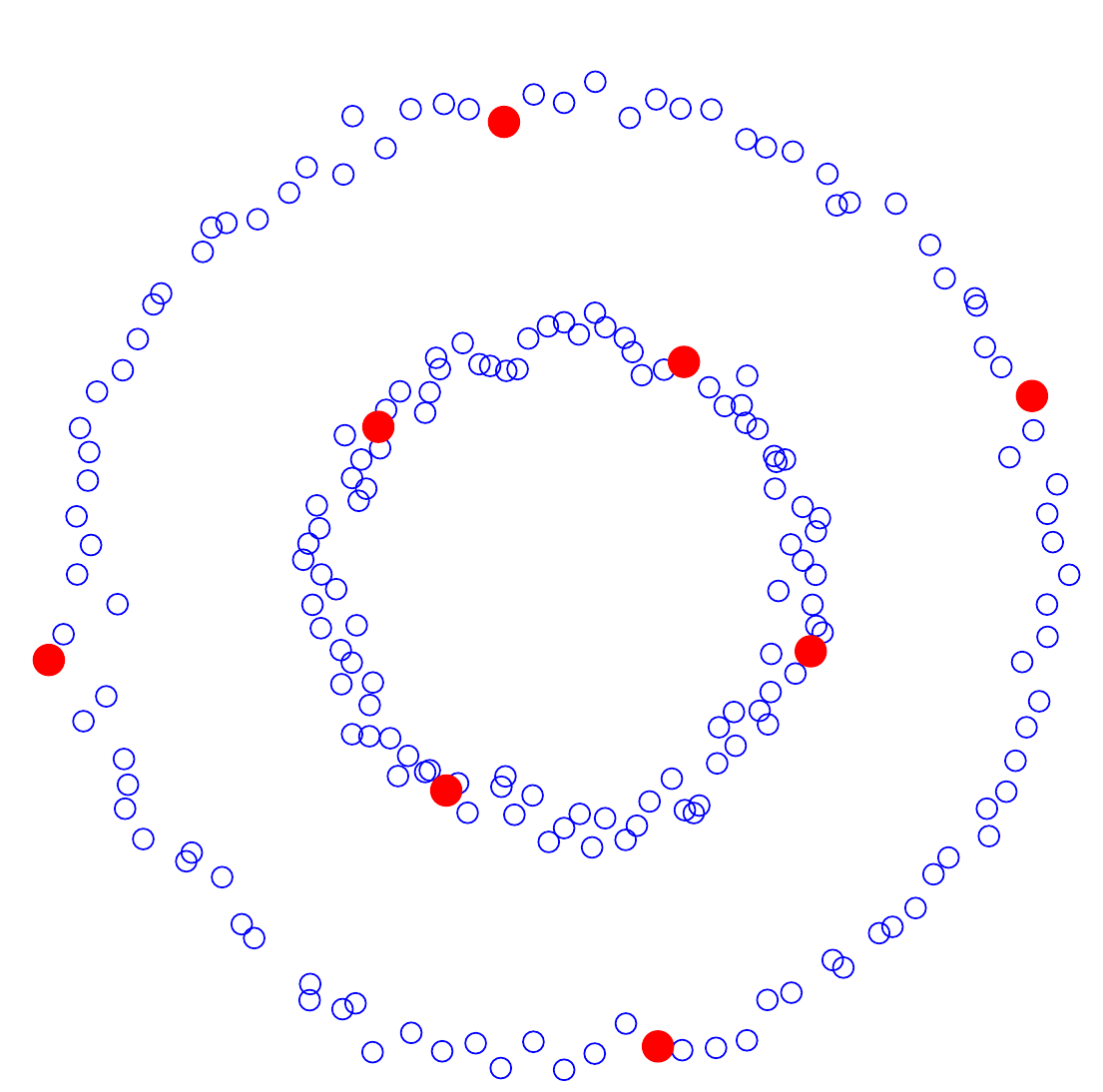}
  \caption{Proposed}
  \label{fig:Proposed_toy}
\end{subfigure}
\caption{Toy example comparing the nodes selected using different active learning methods}
\label{fig:toy}
\end{figure*}

We tested our method in three application scenarios: Handwritten digit recognition, text classification and spoken letters recognition. In these experiments, we do not compare with $\Psi$-max since the method has computational complexity of $O(N^6)$ and, to the best of our knowledge, is not scalable. Next, we provide the details of each experiment. Both the datasets and the graph construction procedures used are typical of what has been used in the literature.



\subsection{Handwritten digits classification}
In this experiment, we used our proposed active semi-supervised learning algorithm to perform a classification task on the USPS handwritten digits dataset\footnote{\url{http://www.cs.nyu.edu/~roweis/data.html}}. This dataset consists of 1100 $16\times 16$ pixel images for each of the digits 0 to 9. We used 100 randomly selected samples for each digit class to create one instance of our dataset. 
Thus each instance consists of 1000 feature vectors (100 samples/class $\times$ 10 digit classes) of dimension 256.

The graph is constructed using Gaussian kernel weights $w_{ij} = \exp\left(- \frac{||\xv_i-\xv_j||^2}{2\sigma^2} \right)$, where $\xv_i$ is the 256-dimensional feature vector composed of pixel intensity values for each image.
The parameter $\sigma$ is chosen to be $1/3$-rd of the average distance to the $K$-th nearest neighbor for all datapoints.
This heuristic has been suggested in \cite{Chapelle-06}.
We fix $K = 10$.
Additionally, the graph is sparsified approximately by restricting the connectivity of each datapoint to its $K$ nearest neighbors, i.e., an edge between nodes $i$ and $j$ is removed unless node $i$ is among the $K$-nearest neighbors of node $j$ or vice-versa.
This results in a symmetric adjacency matrix for the graph.
Using the graph constructed, we select the points to label and report prediction error after reconstruction using our semi-supervised learning algorithm.
We repeat the classification over 10 such instances of the dataset and report the average classification error.
The results are illustrated in Figure~(\ref{fig:usps-results}).
We observe that our proposed method outperforms the others.
A notable feature of our method is that we show very good classification results even for very few labeled samples.
This is due to our inherent criterion for active learning that tries to select those points that maximize the recoverable dimensions of the underlying data manifold.

\subsection{Text classification}
For our text classification example, we use the 20 newsgroups dataset\footnote{\url{http://qwone.com/~jason/20Newsgroups/}}. It contains around 20,000 documents, partitioned in 20 different newsgroups. For our experiment, we consider 10 groups of documents, namely, \{comp.graphics, comp.os.ms-windows.misc, comp.sys.ibm.pc.hardware, comp.\\sys.mac.hardware, rec.autos, rec.motorcycles, sci.crypt, sci.\\electronics, sci.med, sci.space\}, and randomly choose 100 datapoints from each group. We generate 10 such instances of 1000 data points each and report the average errors. We clean the dataset by removing the words that appear in fewer than 20 documents and then select only the 3000 most frequent ones from the remaining words. To form the feature vectors representing the documents, we use the tf-idf statistic of these words. The tf-idf statistic captures the relative importance of a word in a document in a corpus:
\begin{equation}
\text{tf-idf} = (1+\log(\text{tf})) \times \log \left(\frac{N}{\text{idf}} \right)
\end{equation}
where, $\text{tf}$ is the frequency of a word in a document, $\text{idf}$ is the number of documents in which the word appears and $N$ is the total number of documents. Thus, we get $1000$ feature vectors in $3000$ dimensional space. To form the graph of documents, we compute the pairwise cosine similarity between their feature vectors. Each node is connected to the 10 nodes that are most similar to it and the resultant graph is then symmetrized. 
The classification results in Figure~(\ref{fig:text-results}) show that our method performs very well compared to others. However, the absolute error rates are not very good. This is due to the high similarity between different newsgroups which makes the problem inherently difficult.

\subsection{Spoken letters classification}
For the spoken letters classification example, we considered the Isolet dataset\footnote{\url{http://archive.ics.uci.edu/ml/datasets/ISOLET}}. It consists of letters of the English alphabet spoken in isolation twice by 150 different subjects. The speakers are grouped into 5 sets of 30 speakers each, with the groups referred to as isolet1 through isolet5. Each alphabet utterance has been pre-processed beforehand to create a 617-dimensional feature vector.

For this experiment, we considered the task of active semi-supervised classification of utterances into the 26 alphabet categories.
To form an instance of the dataset, 60 utterances are randomly selected out of 300 for each alphabet. Thus, each instance consists of $60\times 26 = 1560$ datapoints of dimension 617.
As in the hand-written digits classification problem, the graph is constructed using Gaussian kernel weights between nodes, with $\sigma$ taken as $1/3$-rd of the average distance to the $K$-th nearest neighbor for each datapoint. We select $K =10$ for our experiment.
Sparsification of the graph is carried out approximately using $K$-nearest neighbor criterion.
With the constructed graph, we perform active semi-supervised learning using all the methods.
The experiment is repeated over 10 instances of the dataset and average prediction error is reported in Figure~(\ref{fig:isolet-results}).
Note that we start with 2\% labeled points to ensure that each method gets a fair chance of selecting at least one point to label from each of the 26 classes.
We observe that our method outperforms the others.
\begin{figure*}
\centering
\begin{subfigure}{0.32\textwidth}
  \centering
  \includegraphics[width=0.98\textwidth]{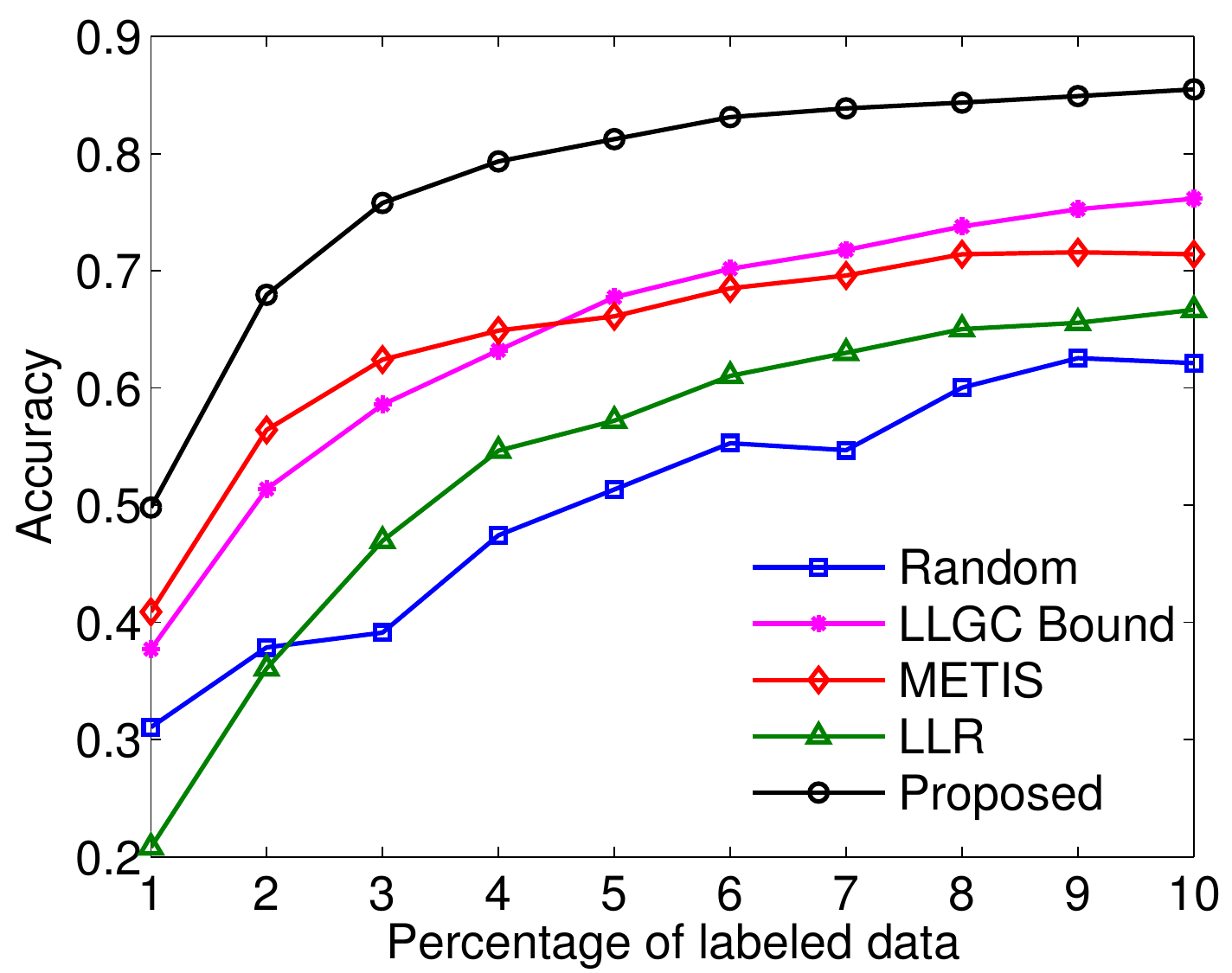}
  \caption{USPS}
  \label{fig:usps-results}
\end{subfigure}
\begin{subfigure}{0.32\textwidth}
  \centering
  \includegraphics[width=0.98\textwidth]{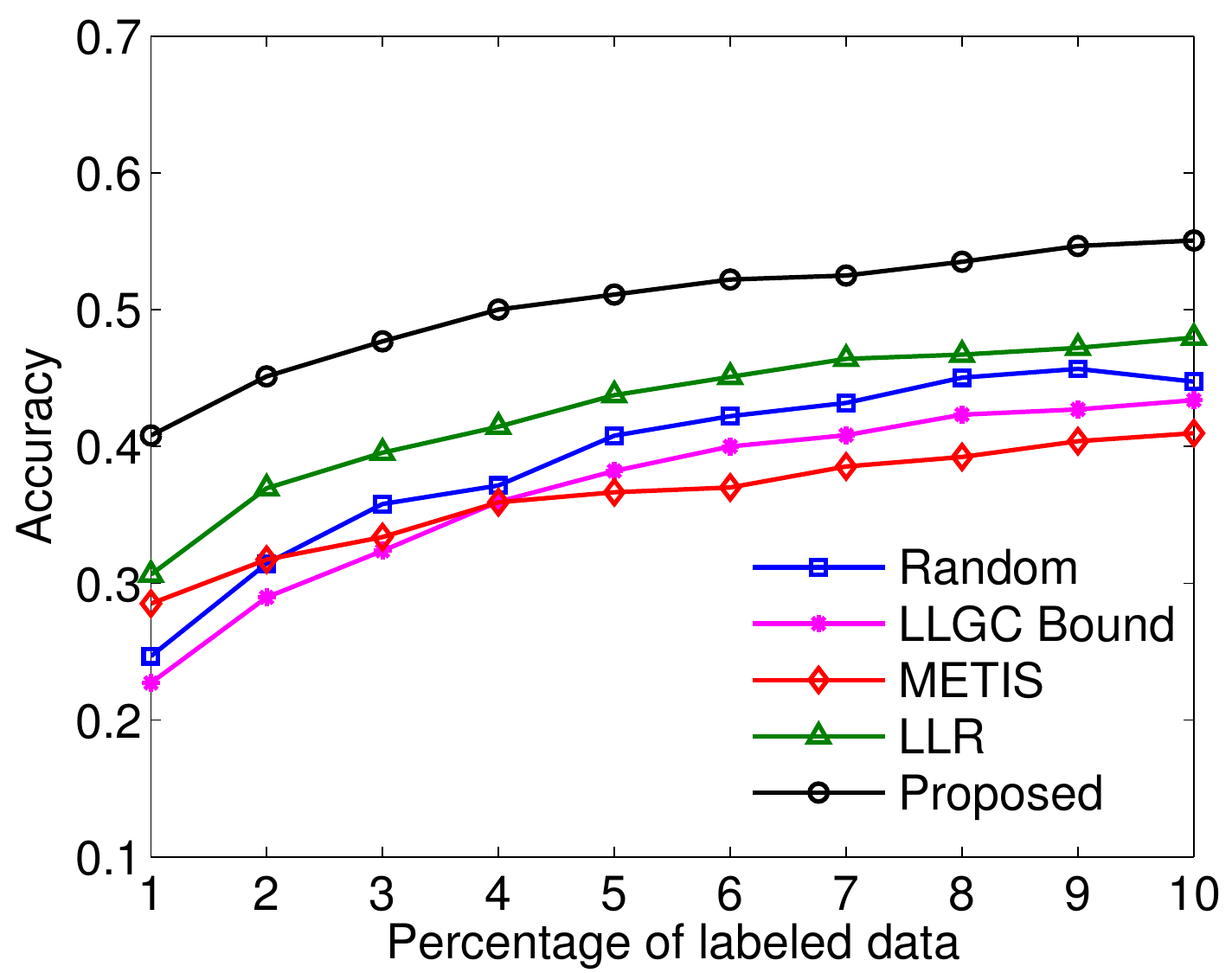}
  \caption{20 newsgroups}
  \label{fig:text-results}
\end{subfigure}
\begin{subfigure}{0.32\textwidth}
  \centering
  \includegraphics[width=0.98\textwidth]{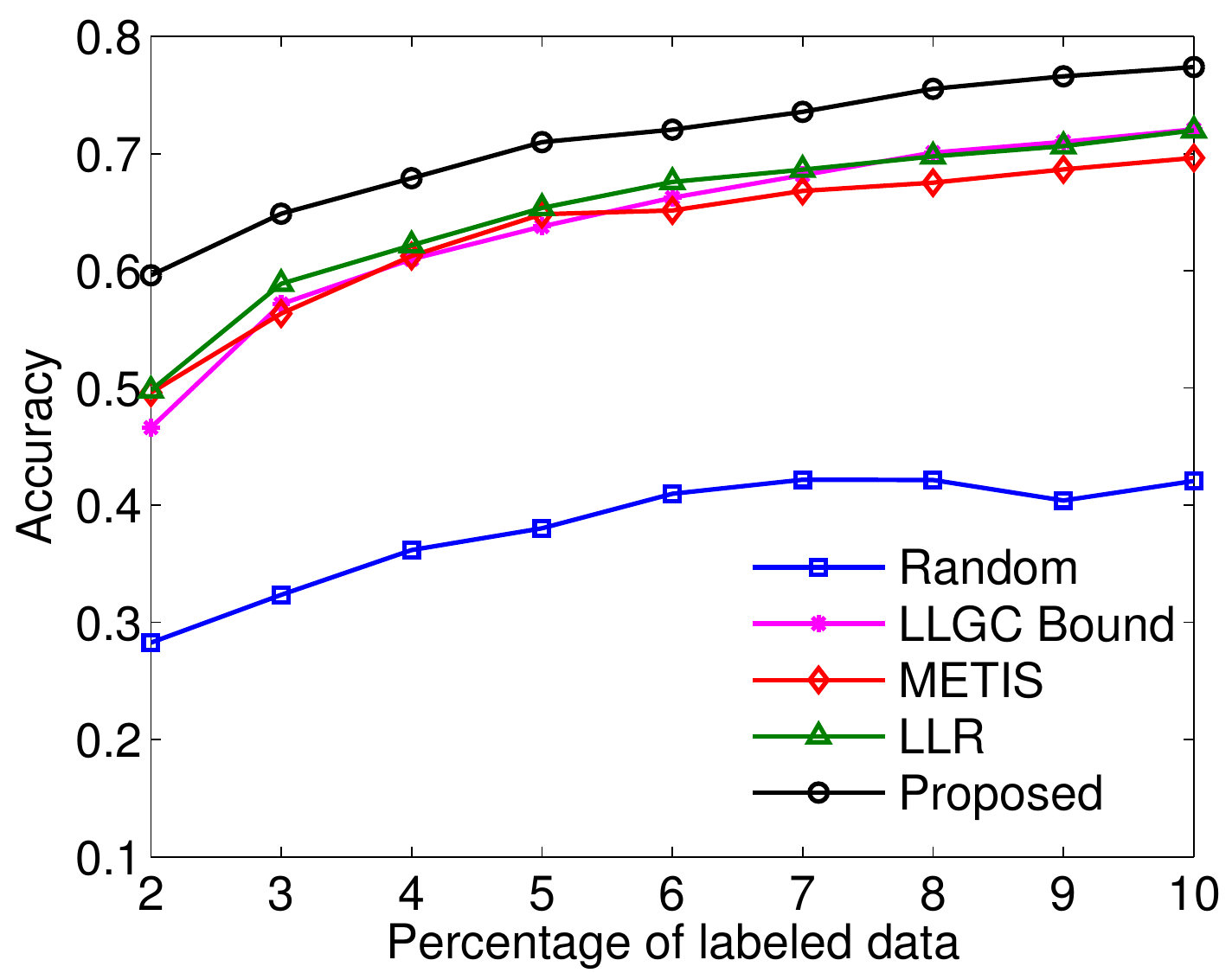}
  \caption{Isolet}
  \label{fig:isolet-results}
\end{subfigure}
\caption{Comparison of active semi-supervised learning methods on real datasets. Plots show the average classification accuracies for different percentages of labeled data.}
\label{fig:results}
\end{figure*}

\subsection{Effect of parameter k}
\label{sec:k_effect}
To study the effect of parameter $k$ in the proposed method on classification accuracy we repeat the above experiments for different values of $k$. Figure~\ref{fig:k_results} shows the results. For the USPS and Isolet datasets, the classification accuracies remain largely unchanged for different values of $k$. For the 20 Newsgroups dataset, a slight improvement in classification accuracies is observed for higher values of $k$. This result agrees with the distribution of GFT coefficients of the class membership functions in each dataset shown in Figure~\ref{fig:gft}. In USPS and Isolet datasets, most of the energy of the graph signal (i.e., the class membership functions) is contained in the first few frequencies. Thus, increasing the value of $k$, so that a better estimate of cut-off frequency is maximized during the choice of sampling set, is not necessary. In other words, maximizing a loose estimate of the cut-off frequency is sufficient. However, the membership functions in the 20 Newsgroups dataset have a significant fraction of their energy spread over high frequencies as shown in Figure~\ref{fig:gft}. Due to this, maximizing a tighter estimate of the the cut-off allows the sampling set selection algorithm to pick nodes that capture more signal energy, resulting in higher accuracies.     
\begin{figure*}
\centering
\begin{subfigure}{0.32\textwidth}
  \centering
  \includegraphics[width=0.98\textwidth]{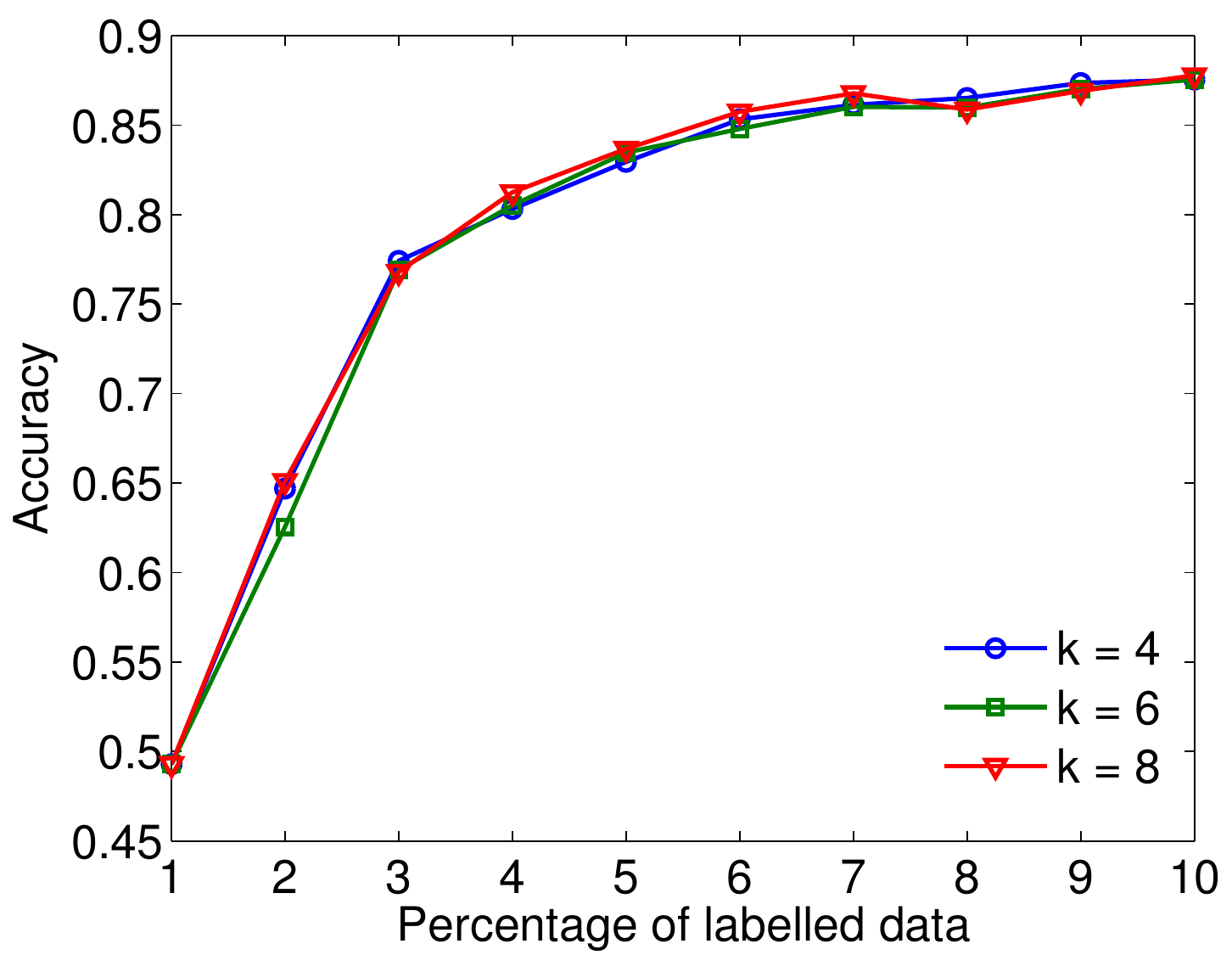}
  \caption{USPS}
  \label{fig:k_usps}
\end{subfigure}
\begin{subfigure}{0.32\textwidth}
  \centering
  \includegraphics[width=0.98\textwidth]{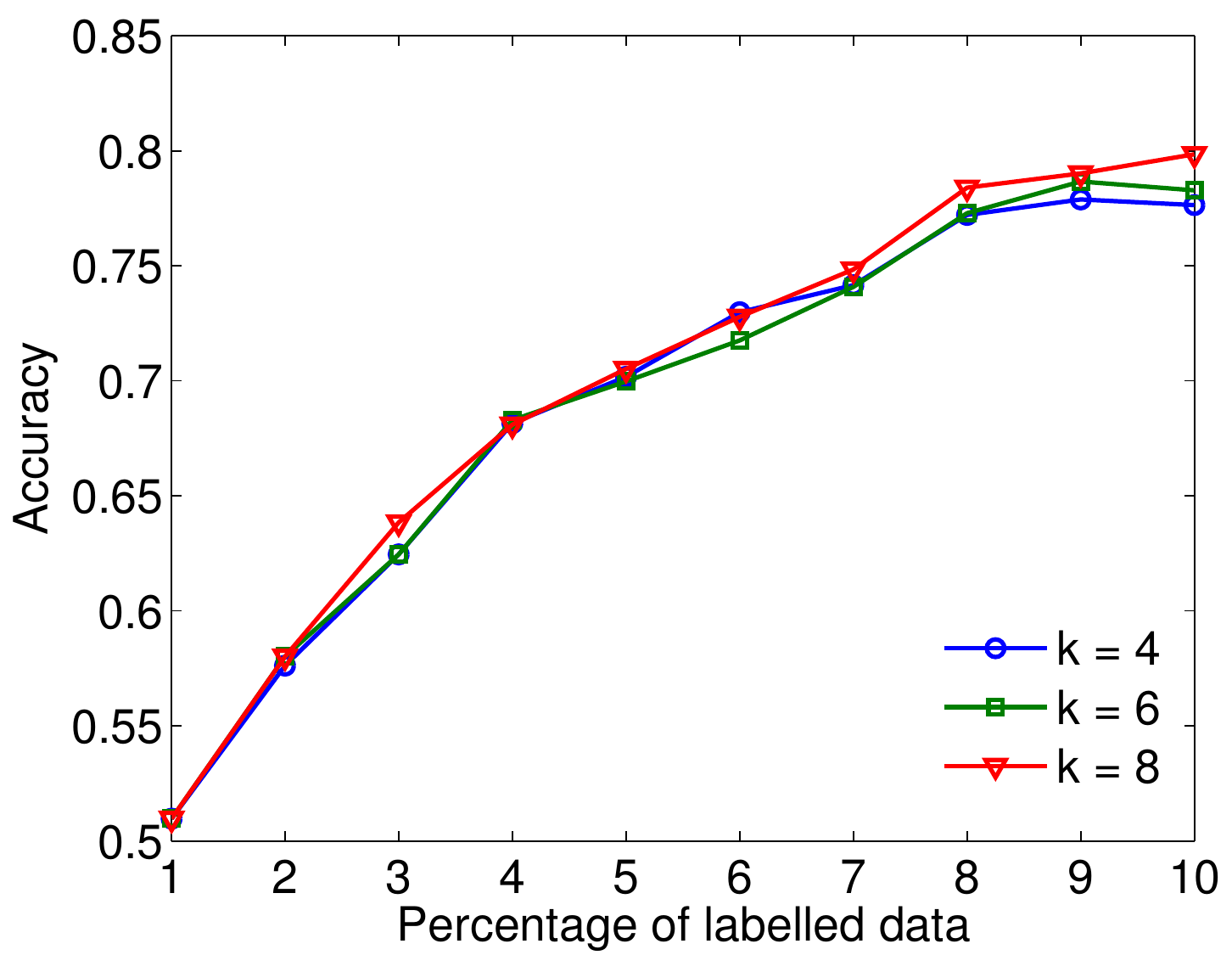}
  \caption{Isolet}
  \label{fig:k_isolet}
\end{subfigure}
\begin{subfigure}{0.32\textwidth}
  \centering
  \includegraphics[width=0.98\textwidth]{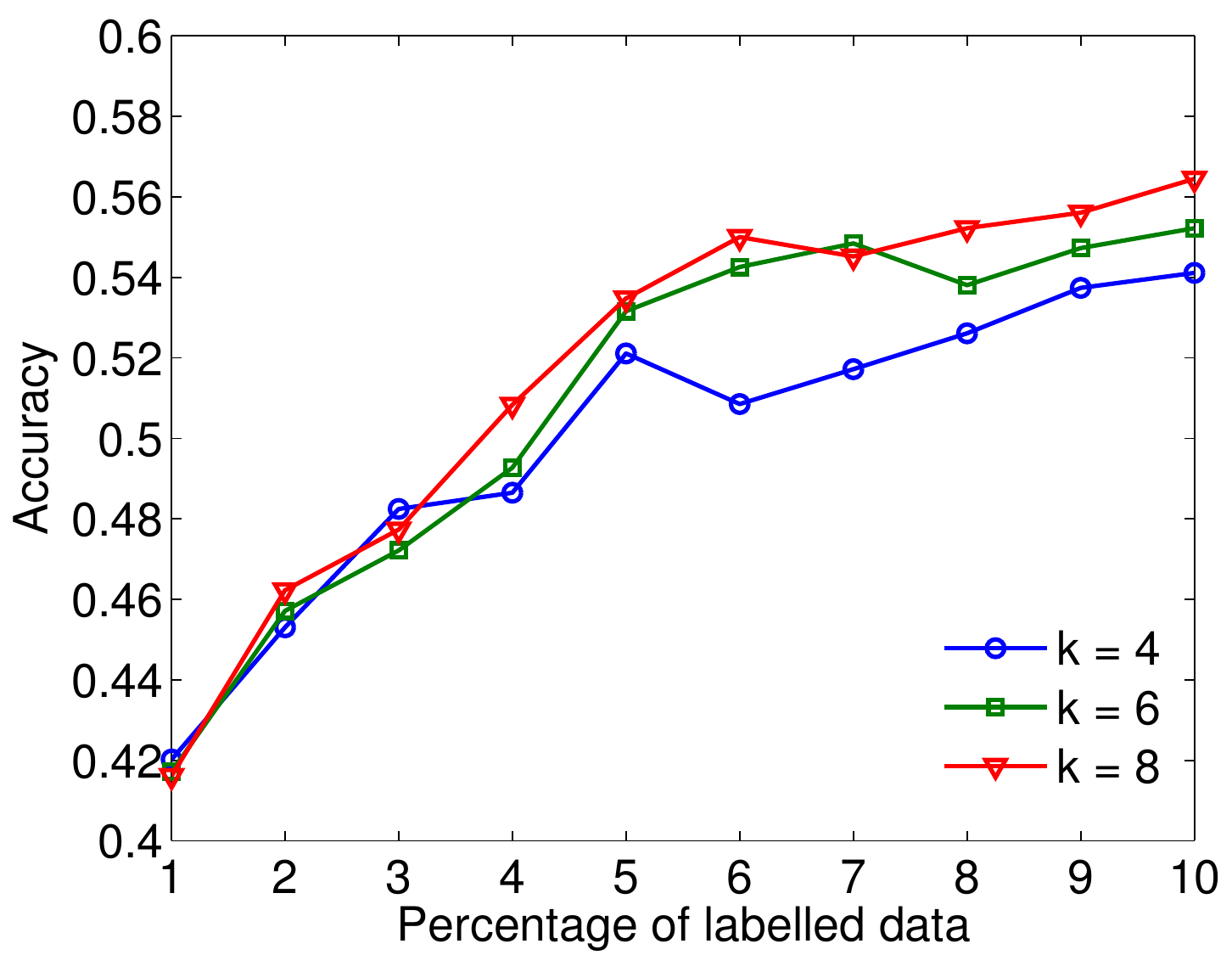}
  \caption{20 newsgroups}
  \label{fig:k_newsgroups}
\end{subfigure}
\caption{Effect of $k$ on classification accuracy of the proposed method. Plots show the average classification accuracy for different percentages of labelled data.}
\label{fig:k_results}
\end{figure*}

\section{Conclusion}
\label{sec:conclusion}
In this paper, we introduce a novel framework for batch mode active semi-supervised learning based on sampling theory for graph signals. The proposed active learning framework aims to select the subset nodes which maximizes the dimension of the space of uniquely recoverable signals. In the context of sampling theory, this translates to selecting the subset with the maximum cut-off frequency. This interpretation leads to a very efficient greedy algorithm. We provide intuition about how the method tries to choose the nodes which are most representative of the data. We also present an efficient semi-supervised learning method based on bandlimited interpolation. We show, through experiments on real data, that our two algorithms, in conjunction, perform very well compared to state of the art methods.       

In the future, we would like to provide bounds on the prediction error of the proposed method (when the true signal is not exactly bandlimited) in terms of signal smoothness and the cut-off frequency. We also hope to have tighter bounds on the number of labels required for desired prediction accuracy. It would be useful to consider an extension of the proposed framework to a partially batch setting so that we can incorporate the label information from previous batches to improve the choice of sampling sets.

%
\bibliographystyle{abbrv}
\bibliography{sigproc_arxiv}  
%
%


\end{document}